\PassOptionsToPackage{table}{xcolor}
\documentclass[sigconf]{acmart}

\usepackage{enumitem}
\usepackage{algorithm}
\usepackage{algpseudocode}
\usepackage{booktabs}
\usepackage{multirow}
\usepackage{graphicx} 
\usepackage{array}
\usepackage{makecell}
\usepackage{subcaption}
\usepackage[table]{xcolor}

\definecolor{rank1}{HTML}{F4B0B0}
\definecolor{rank2}{HTML}{FCE0D4}

\theoremstyle{remark}
\newtheorem{remark}{Remark}

\AtBeginDocument{%
  }

\copyrightyear{2026}
\acmYear{2026}
\setcopyright{cc}
\setcctype{by}
\acmConference[MM '26]{Proceedings of the 34th ACM International Conference on Multimedia}{November 10--14, 2026}{Rio de Janeiro, Brazil}
\acmBooktitle{Proceedings of the 34th ACM International Conference on Multimedia (MM '26), November 10--14, 2026, Rio de Janeiro, Brazil}
\acmDOI{10.1145/3767308.3836106}
\acmISBN{979-8-4007-2213-4/2026/11}

\begin{document}

\title{LEAP: Accelerating ODE-based Diffusion Sampling with Optimal Linear Leaps}

\author{Haowei Zheng}
\orcid{0009-0008-5217-8034}
\email{zhwei7107@zju.edu.cn}

\authornote{State Key Laboratory of Blockchain and Data Security,
Zhejiang University.}
\authornote{Hangzhou High-Tech Zone (Binjiang) Institute of
Blockchain and Data Security.}

\affiliation{%
  \institution{Zhejiang University}
  \city{Hangzhou}
  \country{China}
}

\author{Zhenyu Zhou}
\orcid{0009-0007-0888-5156}
\email{zhyzhou@zju.edu.cn}
\authornotemark[1]
\authornotemark[2]

\affiliation{%
  \institution{Zhejiang University}
  \city{Hangzhou}
  \country{China}
}

\author{Jiawei Chen}
\orcid{0000-0002-4752-2629}
\email{sleepyhunt@zju.edu.cn}
\authornotemark[1]
\authornotemark[2]

\affiliation{%
  \institution{Zhejiang University}
  \city{Hangzhou}
  \country{China}
}

\author{Can Wang}
\orcid{0000-0002-5890-4307}
\email{wcan@zju.edu.cn}
\authornotemark[1]
\authornotemark[2]

\affiliation{%
  \institution{Zhejiang University}
  \city{Hangzhou}
  \country{China}
}

\author{Defang Chen}
\orcid{0000-0003-0833-7401}
\email{defchern@berkeley.edu}
\authornote{Corresponding author.}

\affiliation{%
  \institution{University of California, Berkeley}
  \state{California}
  \country{United States}
}

\renewcommand{\shortauthors}{Haowei Zheng, Zhenyu Zhou, Jiawei Chen, Can Wang, and Defang Chen}

\begin{abstract}
    Diffusion models primarily rely on general-purpose ODE-based numerical methods for training-free accelerated sampling, but performance degrades with coarse time discretization. In this paper, we propose \textit{optimal linear leaps} (LEAP), a specialized method grounded in theoretical and empirical investigations of sampling trajectory geometry.
    Built upon low-dimensionality and cross-trajectory consistency, LEAP formulates a supervised least-squares optimization problem within the local trajectory subspace, achieving a precise approximation of future states without model evaluations. Extensive experiments demonstrate that LEAP outperforms previous methods across diverse modalities, including text-to-image (Stable Diffusion 3.5, FLUX.1-dev) and image-to-3D (TRELLIS.2) generation. Furthermore, LEAP can be seamlessly integrated into existing solver-based methods, offering a plug-and-play enhancement for modern diffusion models. Code is available at \url{https://github.com/zju-pi/LEAP}.
\end{abstract}



\maketitle

\section{Introduction}
Diffusion-based generative models~\cite{ho2020denoising, sohl2015deep, song2020improved,song2020score} have achieved unprecedented success in multimedia content synthesis, ranging from image~\cite{rombach2022high,ramesh2022hierarchical,saharia2022photorealistic}, video~\cite{ho2022video, blattmann2023stable}, audio~\cite{kong2020diffwave,chen2020wavegrad}, and 3D~\cite{xiang2025native,hunyuan3d2025hunyuan3d}. Despite their remarkable generation quality, the slow sampling process imposes a critical bottleneck for latency-sensitive applications.

To address this issue, numerous acceleration techniques have been proposed. Training-based approaches distill a pre-trained multi-step model into an extremely few-step student model~\cite{salimans2022progressive,luo2023latent,song2023consistency,kim2023consistency}. However, distilling modern diffusion models with billions of parameters (e.g., 8B Stable Diffusion 3.5~\cite{esser2024scaling} and 12B FLUX~\cite{flux1dev}) is challenging and inflexible to accommodate community-tuned variants and LoRA modules. In contrast, training-free methods require minimal parameter updates, offering superior plug-and-play flexibility. Recent progress includes \textbf{(i)} solver-based method for calculating ordinary differential equations (ODEs)~\cite{lu2022dpm,song2020denoising,zhao2023unipc,liu2022pseudo,zhou2024fast}, \textbf{(ii)} optimized timestep scheduling~\cite{chen2025geometric,xue2024accelerating,sabour2024align}, and \textbf{(iii)} feature caching~\cite{ma2024deepcache,li2024faster,yu2025ab,wimbauer2024cache}.
Existing high-order multi-step solvers~\cite{liu2022pseudo,lu2022dpm,zhao2023unipc} estimate future states via a linear combination of historical states where the coefficients are derived from numerical analysis, assuming sufficiently small and near-uniform step sizes. However, practical acceleration typically employs large and non-uniform step sizes, making their assumptions and associated convergence guarantees less reliable. Concurrently, emerging research on the geometric regularity of diffusion trajectories~\cite{chen2024trajectory,chen2025geometric,wan2025elucidating,wang2024diffusion} has uncovered intriguing empirical patterns, such as cross-trajectory consistency and low-dimensional concentration. While previous studies have qualitatively observed these regularities, a rigorous mathematical foundation remains largely absent. To move beyond heuristic designs toward a principled approach, a deeper and more rigorous investigation into the underlying mechanics is essential.

\begin{figure*}[ht]
  \centering
  \includegraphics[width=0.97\textwidth]{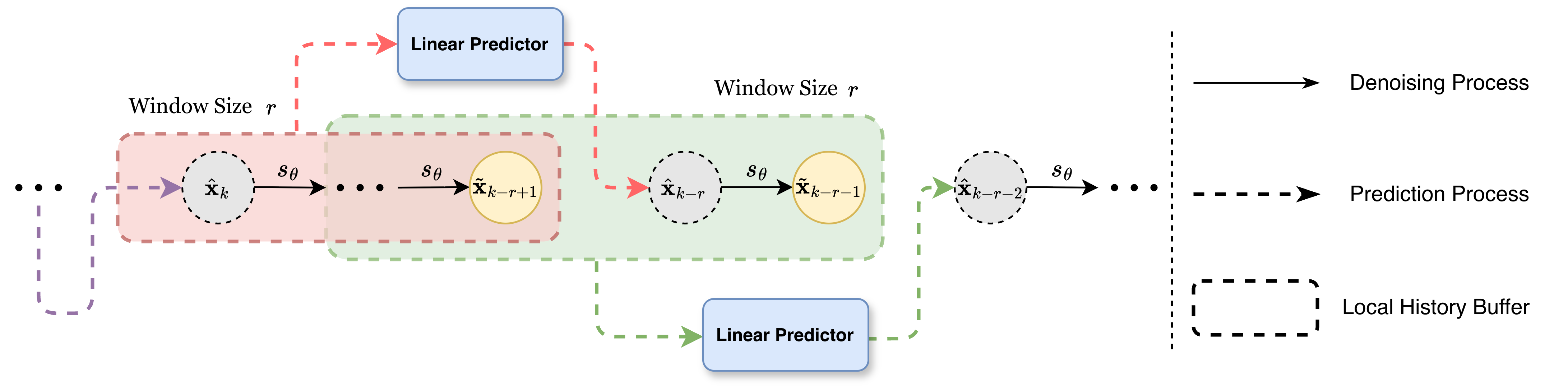}

  \caption{Overview of the proposed LEAP. A linear predictor (blue boxes) extrapolates the subsequent score from a sliding history buffer of size $r$ (dashed boxes), yielding an approximated state $\hat{\mathbf{x}}$ (grey circles) without model evaluations. To prevent error drift, a model evaluation subsequently refines this approximation into a corrected state $\tilde{\mathbf{x}}$ (yellow circles). }

  \Description{Overview of the proposed prediction-correction sampling. A linear predictor (blue boxes) extrapolates the next score from a sliding history buffer of size $r$ (dashed boxes), yielding an approximated state $\hat{\mathbf{x}}$ (grey circles) without model evaluations. To prevent error drift, an actual model evaluation subsequently refines this approximation into a corrected state $\tilde{\mathbf{x}}$ (yellow circles). }

  \label{fig:pipeline}

\end{figure*}

To bridge this gap, we conduct a comprehensive theoretical analysis of the geometric properties inherent in diffusion trajectories, yielding the following two theoretical insights. Firstly, we establish a theoretical bound on the alignment error between trajectories after orthogonal rotational alignment, revealing the cross-trajectory consistency of diffusion models. Secondly, we rigorously characterize the intrinsic low-dimensionality of diffusion sampling trajectories, proving that individual paths are confined to a low-dimensional subspace. Our analysis reveals that despite different noise initializations, these generative paths share consistent low-dimensional dynamics.

Leveraging these two mathematically grounded priors, we propose \textit{optimal linear leaps} (LEAP), a training-free, solver-based acceleration framework. In contrast to traditional numerical solvers that rely on rigid Taylor coefficients, our method introduces a data-driven multi-step strategy powered by least-squares optimization. 
The intrinsic low-dimensionality of diffusion trajectories suggests that local path segments can be accurately characterized by a linear subspace. Leveraging this geometric regularity, we formulate the prediction as a least-squares optimization problem, which identifies the optimal linear coefficients by minimizing the discrepancy between linearly extrapolated historical states and multi-step ground-truth trajectory samples.
Motivated by the cross-trajectory consistency, we demonstrate that the optimal linear extrapolation coefficients are sample-agnostic. This property guarantees the transferability of coefficients derived from a limited set of reference trajectories before inference, bypassing the overhead for per-sample optimization.
LEAP significantly enhances the sampling efficiency while preserving high generation quality. Crucially, as it operates entirely at the macroscopic trajectory level, our approach functions as a model-agnostic and plug-and-play module. This allows for seamless integration into diverse diffusion models without any architectural modifications.
Extensive experiments across diverse modalities, including text-to-image generation and text-to-3D generation demonstrate the superior performance and versatility of our method.

In summary, our main contributions are as follows:

\begin{itemize}[leftmargin=1.2em]
    \item We prove the intrinsic low-dimensionality and geometric isomorphism of generative trajectories across diverse noise initializations.
    \item Motivated by these theoretical insights, we propose LEAP, a novel multi-step ODE solver that leverages least-squares alignment to learn sample-agnostic extrapolation coefficients, enabling highly accurate linear leaps with minimal computational overhead.
    \item Extensive experiments across diverse generative tasks that spanning both text-to-tmage synthesis and complex image-to-3D generation demonstrate the effectiveness of our method compared to state-of-the-art training-free baselines.
\end{itemize}

\section{Preliminary}
\textit{\textbf{Diffusion Models}}. Continuous-time diffusion models progressively corrupt data $\mathbf{x}_0 \sim p_{data}(\mathbf{x})$ into a tractable noise distribution $\mathbf{x}_T \sim p_{prior}(\mathbf{x})$ via a forward Stochastic Differential Equation (SDE)~\cite{song2020score}:
\begin{equation}
    \label{eq:forward_sde}
    \mathrm{d}\mathbf{x}_t=f(t)\mathbf{x}_t\mathrm{d}t+g(t)\mathrm{d}\mathbf{w}_t, \qquad t \in [0, T],
\end{equation}
where $\mathbf{w}_t$ is the standard Wiener process~\cite{oksendal2013stochastic}, while $f(t)$ and $g(t)$ denote the drift and diffusion coefficients. 
Crucially, the generative process that reverses the forward SDE to recover the data distribution is governed by the reverse-time SDE $\mathrm{d}\mathbf{x}_t=[f(t)\mathbf{x}_t-g(t)^2\nabla_{\mathbf{x}_t}\log p_t(\mathbf{x}_t)]\mathrm{d}t+g(t)\mathrm{d}\bar{\mathbf{w}}_t$,
where $\bar{\mathbf{w}}_t$ denotes a standard Wiener process in the reverse-time direction. 
To implement the generative process, diffusion models estimate the intractable \textit{score function} $\nabla_{\mathbf{x}_t} \log p_t(\mathbf{x}_t)$ using a neural network $s_\theta(\mathbf{x}_t, t)$ trained by score matching objectives.
For any diffusion SDE, there exists an Ordinary Differential Equation (ODE)~\cite{song2020score} that shares the identical marginal distributions $p_t(\mathbf{x})$ for all $t \in [0, T]$, which is known as the Probability Flow ODE (PF-ODE):
\begin{equation}
    \mathrm{d}\mathbf{x}_t=\left[f(t)\mathbf{x}_t-\frac{1}{2}g(t)^2\nabla_{\mathbf{x}_t}\log p_t(\mathbf{x}_t)\right]\mathrm{d}t.
\end{equation}
By casting the generative process as a numerical initial value problem, PF-ODE enables deterministic sampling and underpins various training-free solver-based acceleration methods~\cite{lu2022dpm,zhao2023unipc}.

\textit{\textbf{Flow-based Models}}. 
Instead of explicitly defining a diffusion process through an SDE, flow-based generative models formulate a continuous-time Normalizing Flow~\cite{lipman2022flow,liu2022flow,albergo2022building}. These models learn a time-dependent vector field $v(\mathbf{x}_t, t)$ that transports noise to the data distribution via a deterministic ODE:
\begin{equation}
    \mathrm{d}\mathbf{x}_t=v(\mathbf{x}_t,t)\mathrm{d}t.
\end{equation}
To circumvent the direct estimation of intractable marginal vector fields, Flow Matching introduces a simulation-free objective which regresses a neural vector field $v_\theta(\mathbf{x}_t, t)$ towards a conditional vector field $\mathbf{u}_t(\mathbf{x}_t|\mathbf{x}_0)$ that defines the target probability path $p_t(\mathbf{x})$:
\begin{equation}
    \mathbb{E}_{t,p_0(\mathbf{x}_0),p_t(\mathbf{x}_t\mid \mathbf{x}_0)}
    \left[
    \left\|v_\theta(\mathbf{x}_t,t)-\mathbf{u}_t(\mathbf{x}_t\mid \mathbf{x}_0)\right\|_2^2
    \right].
\end{equation}
Specifically, by setting $T = 1$, analytically solving Eq. \ref{eq:forward_sde} yields a reparameterization of $\mathbf{x}_t$ as
\begin{equation}
    \mathbf{x}_t=\alpha_t \mathbf{x}_0 + \sigma_t \mathbf{x}_1, \qquad \mathbf{x}_1 \sim \mathcal{N}(\mathbf{0},\mathbf{I}).
\end{equation}
where
\begin{equation}
    \small
    \alpha_t=\exp\left(\int_0^t f(u)\mathrm{d}u\right),\quad
    \sigma_t=\left(\int_0^t g(u)^2 \exp\left(-2\int_0^u f(s)\mathrm{d}s\right)\mathrm{d}u\right)^{1/2}.
\end{equation}
The conditional velocity field is thus given by
\begin{equation}
    \mathbf{u}_t(\mathbf{x}_t\mid \mathbf{x}_0)
    =
    \frac{\dot{\sigma}_t}{\sigma_t}\mathbf{x}_t
    +
    \left(
    \dot{\alpha}_t-\alpha_t\frac{\dot{\sigma}_t}{\sigma_t}
    \right)\mathbf{x}_0.
\end{equation}

Contemporary cutting-edge text-to-image models increasingly favor the Flow Matching framework due to its conceptual simplicity and flexible path design~\cite{esser2024scaling,flux1dev}. Since they inherit the architecture and training conventions of diffusion models, we still refer to the flow-based formulation as ``diffusion'' when no ambiguity arises.

\textit{\textbf{Trajectory Regularity}}. While diffusion models operate in an immensely high-dimensional ambient space, recent empirical observations reveal that their sampling trajectories exhibit remarkable geometric regularity~\cite{chen2024trajectory,wang2023diffusion,chen2025geometric}. Principal Component Analysis (PCA) on individual sampling trajectory exhibits that the generative dynamics are inherently \emph{low-dimensional} where the explained variance is dominated by the first few principal components. Furthermore, as illustrated in Fig.~\ref{fig:PCA_trajectory}, when projecting multiple trajectories with distinct noise and prompt initializations into a shared 3D space, they bundle into a smooth `tube'. This visual alignment highlights a striking \emph{cross-trajectory consistency}, indicating that different diffusion paths share nearly isomorphic macroscopic dynamics. Together, this inter-trajectory low-dimensionality and inter-trajectory consistency constitute the trajectory regularity, serving as the foundational intuition for our subsequent theoretical analysis and the design of our acceleration algorithm.

\begin{figure}[t]
    \centering
    \includegraphics[width=\linewidth]{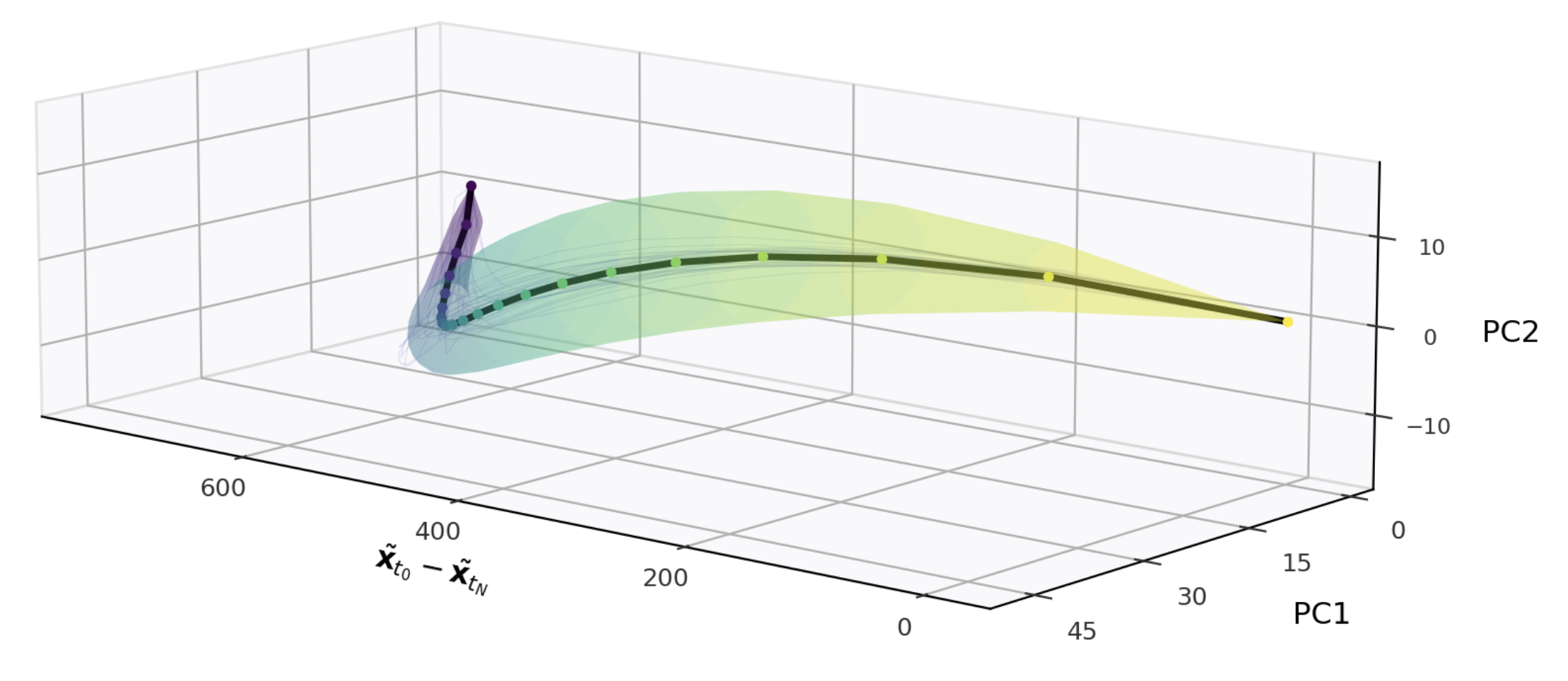}
    \caption{Visualizing the trajectory regularity of diffusion models. Following the visualization protocol in~\cite{chen2024trajectory}, we project 100 sampling trajectories generated with distinct noise and prompt initializations using SD3.5-Medium~\cite{esser2024scaling}. 
    }
    \label{fig:PCA_trajectory}
\end{figure}

\section{Methodology}
\subsection{The Intrinsic Geometry of Generative Trajectories}\label{sec:geometry}

While previous works~\cite{chen2024trajectory,chen2025geometric} have observed strong geometric regularity of ODE sampling trajectories to justify existing accelerated sampling methods, a theoretical understanding remains unclear. In this section, we prove that diffusion sampling trajectories can be effectively aligned through orthogonal transformation and are embedded in low-dimensional subspaces. To this end, we perform an analysis within the Gaussian mixture model (GMM) framework with shared spherical covariance to analytically characterize diffusion dynamics. Since GMMs are universal approximators capable of modeling highly complex data manifolds, this assumption not only facilitates closed-form mathematical derivations but also ensures the broad applicability of our theoretical findings. Under this mild assumption, the data distribution is given by
\begin{equation}
p_{\text{data}}(\mathbf{x})=\sum_{s=1}^S \pi_s \mathcal N(\boldsymbol{\mu}_s,\sigma_0^2 \mathbf{I}_d),
\quad
\pi_s>0,
\quad
\sum_{r=1}^S \pi_s=1,
\end{equation}
where $S$ is the number of components, $\pi_s$ are the mixing weights and $\boldsymbol{\mu}_s$ are the cluster means. Let
$$U:=\operatorname{span}\{\boldsymbol{\mu}_1,\dots,\boldsymbol{\mu}_S\},\quad
r:=\dim(U),\quad
R:=\max_{1\le s\le S}\|\boldsymbol{\mu}_s\|.$$

We consider the exact PF-ODE associated with the variance-preserving forward diffusion. Let $t \in [0, T]$ denote the standard forward diffusion time. Denote $\mathbf{x}_t^{(m)}$ the $m$-th sampling trajectory evolving backward in time which is initialized from a pure noise state $\mathbf{x}_T^{(m)} \sim p_T$. To analyze the trajectory at discrete intervals, we consider a backward time grid $T = t_N > t_{N-1} > \dots > t_0 = 0$ and define the trajectory states as
\begin{equation}
\mathbf{D}_n^{(m)} := \mathbf{x}_{t_{n}}^{(m)} - \mathbf{x}_T^{(m)},
\end{equation}
which represents the displacement from the initial noise. Let
\begin{equation}
\alpha_T:=\exp\!\Big(-\frac12\int_0^T \beta(s)\,ds\Big),
\qquad
\nu_T^2:=\alpha_T^2\sigma_0^2 + (1-\alpha_T^2).
\end{equation}

For brevity, we present a simplified statement with a proof sketch below. The complete theorem including all constants and rigorous proof is detailed in the Appendix A.

\begin{theorem}[Cross-trajectory consistency]\label{thm:main_consistency}Suppose that $\mathbf{x}_T^{(1)},\dots,\mathbf{x}_T^{(M)}$ are i.i.d. samples from the exact noised marginal $p_T$. Then for every $\delta\in(0,1)$, with probability at least $1-\delta$,
\begin{multline}
\max_{1\le m_1,m_2\le M}\; \inf_{\mathbf{H}\in O(d)} \max_{0\le n\le N} \bigl\| \mathbf{D}_n^{(m_1)} - \mathbf{H} \mathbf{D}_n^{(m_2)} \bigr\| \\
\le C_{T,1}\,\nu_T\!\left(\sqrt{r}+\sqrt{\log\frac{M}{\delta}}\right) + C_{T,2}\,R,
\end{multline}
where $C_{T,1},C_{T,2}>0$ depend only on the diffusion schedule and not on $d$, $n$, or the mixture parameters.\end{theorem}

\begin{corollary}[Asymptotic consistency]\label{cor:main_consistency}If $r+\log M = o(d)$ and $R = o(\nu_T \sqrt{d})$, then for every fixed $\delta\in(0,1)$,
\begin{equation}\frac{1}{\nu_T\sqrt d}
\max_{1\le m_1,m_2\le M}\;
\inf_{\mathbf{H}\in O(d)}
\max_{0\le n\le N}
\bigl\| \mathbf{D}_n^{(m_1)} - \mathbf{H} \mathbf{D}_n^{(m_2)} \bigr\|
\;\xrightarrow[d\to\infty]{\mathbb P}\; 0.
\end{equation}
\end{corollary}

\begin{proof}[Proof sketch] Let $P$ and $Q=\mathbf{I}-P$ denote the orthogonal projections onto $U$ and $U^\perp$, respectively. 
In the Appendix A, we show that the reverse-time PF-ODE can be written as
\begin{equation}
\mathrm{d}\mathbf{x}_t = \bigl[ a_t \mathbf{x}_t + b_t \mathbf{m}_t(P \mathbf{x}_t) \bigr] \mathrm{d}t,
\end{equation}
where $\mathbf{m}_t(\cdot)\in U$ is the posterior mean of the component centers at time $t$. Consequently, separating the dynamics yields
\begin{equation}
Q\,\mathrm{d}\mathbf{x}_t = a_t Q\mathbf{x}_t\,\mathrm{d}t,
\qquad
P\,\mathrm{d}\mathbf{x}_t = \bigl[ a_t P\mathbf{x}_t + b_t \mathbf{m}_t(P\mathbf{x}_t) \bigr] \mathrm{d}t.
\end{equation}
Therefore, the $U^\perp$-component evolves by a common scalar rescaling, while all mixture-specific non-linear effects are confined to the $r$-dimensional subspace $U$.

Geometrically, this implies that after applying the orthogonal alignment, the maximum discrepancy between any pair of trajectories is determined by two factors: the norm difference of the initial noise within the orthogonal complement (i.e., $\big|\|Q\mathbf{x}_T^{(m_1)}\| - \|Q\mathbf{x}_T^{(m_2)}\|\big|$), and the absolute magnitude of the signal component within the data subspace (i.e., $\|P\mathbf{x}_T^{(m)}\|$). 
Since $Q\mathbf{x}_T^{(m)}$ is a Gaussian vector in dimension $d-r$, its norm concentrates uniformly over $m$, such that the norm difference is bounded by $O\bigl(\nu_T \sqrt{\log(M/\delta)}\bigr)$. On the other hand, $P\mathbf{x}_T^{(m)}$ lives in an $r$-dimensional subspace and its norm satisfies
\begin{equation}
\|P\mathbf{x}_T^{(m)}\| = O\left(\alpha_T R + \nu_T\Bigl(\sqrt r + \sqrt{\log(M/\delta)}\Bigr)\right)
\end{equation}
uniformly over $m$ with high probability. Combining these bounds yields Thm.~\ref{thm:main_consistency}, and Cor.~\ref{cor:main_consistency} follows by normalization with $\nu_T\sqrt d$.\end{proof}

\textbf{Remark}. Cor.~\ref{cor:main_consistency} indicates that in high-dimensional settings ($d \to \infty$) which matches setting of diffusion models, the error bound in Thm.~\ref{thm:main_consistency} becomes practically effectively vanishes. This provides a strong theoretical guarantee that our cross-trajectory alignment is highly accurate in realistic generation scenarios.

\begin{corollary}[Intrinsic low-dimensionality of individual trajectories]
\label{cor:trajectory_low_dim}
Under the assumptions of Thm.~\ref{thm:main_consistency}, for each sampling trajectory \(\mathbf{x}_t^{(m)}\), there exists a linear subspace
\begin{equation}
\mathcal S_m := U \oplus \operatorname{span}\{Q\mathbf{x}_T^{(m)}\}
\end{equation}
with \(\dim(\mathcal S_m)\le r+1\), such that
\begin{equation}
\mathbf{x}_{t_n}^{(m)} \in \mathcal S_m,
\qquad
0\le n\le N.
\end{equation}
Equivalently, if $\mathbf{X}^{(m)}\in\mathbb R^{(N+1)\times d} $denotes the trajectory matrix whose \(n\)-th row is \((\mathbf{x}_{t_n}^{(m)})^\top\), then
\begin{equation}
\operatorname{rank}(\mathbf{X}^{(m)})\le r+1.
\end{equation}
\end{corollary}

\textbf{Remark}. Cor.~\ref{cor:trajectory_low_dim} reveals that all trajectory variance is captured by at most \(r+1\) principal components. However, this \((r+1)\)-dimensional subspace \(\mathcal S_m\) is sample-dependent, as it is spanned by the initial noise vector \(Q\mathbf{x}_T^{(m)}\). The following corollary establishes a remarkable geometric invariance: by applying a suitable orthogonal transformation to each trajectory, all trajectories can be perfectly aligned into a \textbf{single, shared} low-dimensional subspace.

\begin{corollary}[Universal low-dimensional alignment of diffusion trajectories]
\label{cor:low_dimensional_trajectory_alignment}
For any fixed unit vector \(\mathbf{e}\in U^\perp\), there exist orthogonal matrices \(\mathbf{H}_1,\dots,\mathbf{H}_M\in O(d)\) satisfying \(\mathbf{H}_m|_U = \mathbf{I}\) and
\begin{equation}
\mathbf{H}_m Q\mathbf{x}_T^{(m)}=\|Q\mathbf{x}_T^{(m)}\|\,\mathbf{e}, \qquad \forall m \in \{1, \dots, M\}.
\end{equation}
Under such transformations, the aligned trajectories defined by
\begin{equation}
\bar{\mathbf{x}}_{t_n}^{(m)}:=\mathbf{H}_m\mathbf{x}_{t_n}^{(m)}
\end{equation}
all reside in the same \((r+1)\)-dimensional subspace
\begin{equation}
\bar{\mathcal S}:=U\oplus \operatorname{span}\{\mathbf{e}\}.
\end{equation}
\end{corollary}

PF-ODE exhibits two complementary structural properties: cross-trajectory consistency and intrinsic individual low-dimensionality. Thm.~\ref{thm:main_consistency} establishes that, different sampling trajectories can be highly aligned through simple orthogonal transformation. Cor.~\ref{cor:trajectory_low_dim} further demonstrates that each trajectory is strictly confined to an \((r+1)\)-dimensional subspace. Notably, after a simple per-trajectory alignment introduced by Cor.~\ref{cor:low_dimensional_trajectory_alignment}, the entire collection of trajectories resides within a universal \((r+1)\)-dimensional subspace. Therefore, the high-dimensional sampling dynamics are effectively governed by only a few latent coordinates: one radial coordinate in \(U^\perp\) and at most \(r\) coordinates within the mean subspace \(U\).

\begin{figure}[t]
    \centering
    \includegraphics[width=\linewidth]{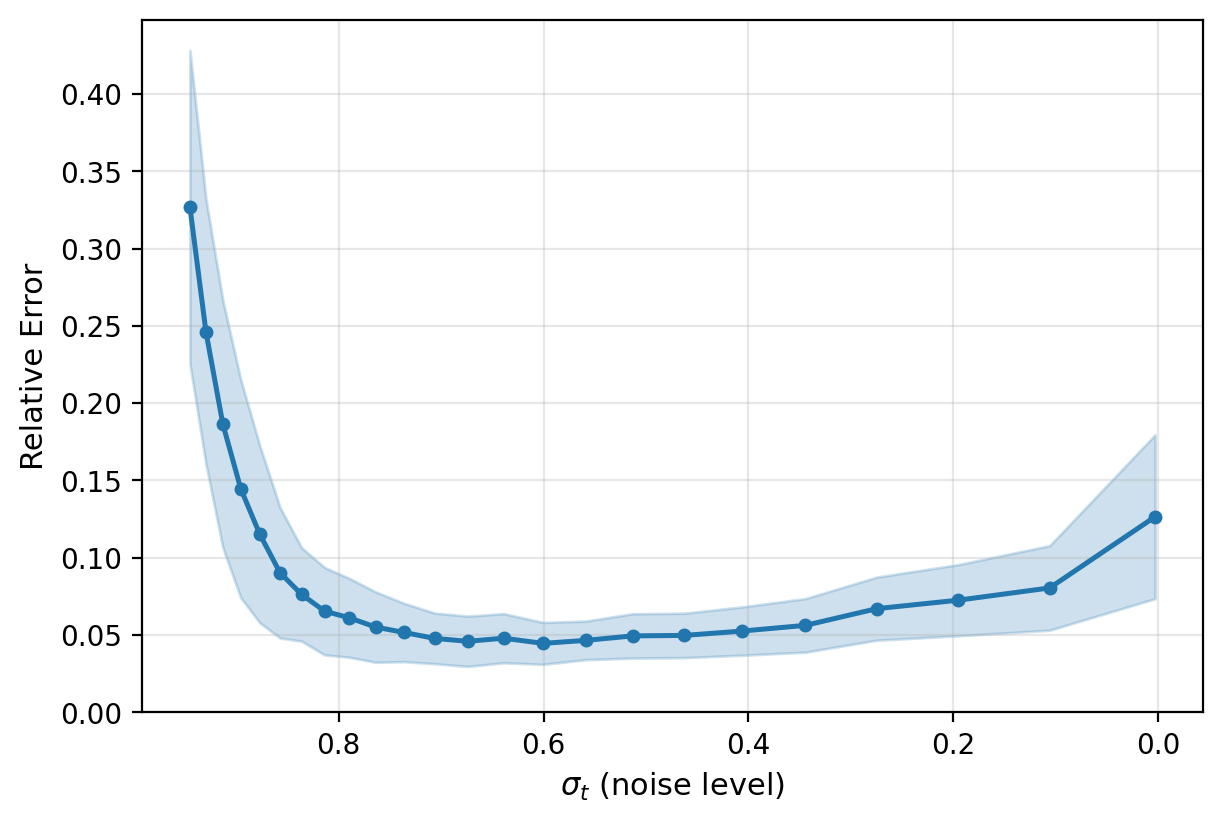}
    \caption{Relative prediction error vs. noise level $\sigma_t$. We quantify the projection error of each score onto the subspace spanned by its corresponding historical buffer. The relative error is strictly bounded, using a small window size of $r=3$.}
    \label{fig:empirical_low_dimensional}
\end{figure}

More importantly, the intrinsic low-dimensionality established in Cor.~\ref{cor:trajectory_low_dim} implies that, within any local window of a given trajectory, both the historical scores and the target output reside in a subspace. While this geometric regularity is rigorously established under mild assumption, its validity in real-world generative scenarios warrants empirical verification. To bridge this gap, we investigate the local score dynamics using sampling trajectories from Stable Diffusion 3.5. Specifically, we quantify the orthogonal projection error of the current score $\mathbf{s}_t$ onto the local subspace spanned by a short historical buffer, denoted as $\text{span}(\mathbf{s}_{t+r}, \mathbf{s}_{t+r-1}, \ldots, \mathbf{s}_{t+1})$. As illustrated in Fig.~\ref{fig:empirical_low_dimensional}, the relative projection error is strictly bounded, dropping to a mere 4\%–5\% even with a small window size $r=3$. This empirical finding confirms that even for highly complex data distributions, the immediate future score remains strictly confined within the linear span of its recent history. This empirical evident for low-dimensionality seamlessly transforms our theoretical insights into a practical algorithm, leading to our linear extrapolation framework in the subsequent section.

\subsection{Optimal Linear Leaps with Aligned Trajectories}

In this section, we present our method grounded in the the local linear dependence established above. This geometric regularity underpins the use of a linear extrapolator to predict next score function. The overall pipeline of our approach is illustrated in Fig.~\ref{fig:pipeline}. To formalize this process, suppose we collect an ensemble of $M$ fine-grained reference sampling trajectories. For the $m$-th trajectory ($m=1,\ldots,M$), denote $\tilde{\mathbf{s}}_n^{(m)}$ the model output (e.g., estimated score function) at timestep $t_n$, and let $\tilde{\mathbf{x}}_{n-1}^{(m)}$ be the next state updated using $\tilde{\mathbf{s}}_n^{(m)}$.

Given a window size $r$, our objective at timestep $t_n$ is to predict the next model output $\tilde{\mathbf{s}}_{n-1}$ by utilizing $r$ stored previous model outputs along the sampling trajectory. Denote the local history matrix of scores and the corresponding target by
\begin{equation}
\mathbf{H}_n^{(m)}=
\left[
\tilde{\mathbf{s}}_{n+r-1}^{(m)},
\tilde{\mathbf{s}}_{n+r-2}^{(m)},
\ldots,
\tilde{\mathbf{s}}_{n}^{(m)}
\right]\in\mathbb{R}^{d\times r}.
\end{equation}
We seek an optimal coefficient vector $\mathbf{w}_n\in\mathbb{R}^r$ that satisfies:
\begin{equation}
\tilde{\mathbf{s}}_{n-1}^{(m)} \approx \mathbf{H}_n^{(m)}\mathbf{w}_n.
\end{equation}
The cross-trajectory consistency presented in Thm.~\ref{thm:main_consistency} demonstrates that, after proper alignment, different trajectories follow nearly isomorphic low-dimensional dynamics. Consequently, a \emph{single} coefficient vector $\mathbf{w}_n$ can effectively characterize the score evolution of different reference trajectories at a given sampling timestep, enabling the extraction of a universal local extrapolation law. Formally, we derive the step-wise optimal predictor within the family of linear extrapolators by solving the following multi-trajectory least-squares problem:
\begin{equation}
\mathbf{w}_n^\star
=
\arg\min_{\mathbf{w}\in\mathbb{R}^r}
\sum_{m=1}^M
\left\|
\tilde{\mathbf{s}}_{n-1}^{(m)}-\mathbf{H}_n^{(m)}\mathbf{w}
\right\|_2^2.
\end{equation}

\begin{algorithm}[t]
\caption{Coefficient Bank Construction}
\label{alg:bank_construction}
\begin{algorithmic}[1]
\Require Reference score trajectories $\{\mathcal{T}^{(m)}\}_{m=1}^M$, where $\mathcal{T}^{(m)}=\{\tilde{\mathbf{s}}_{N}^{(m)},\tilde{\mathbf{s}}_{N-1}^{(m)},\ldots,\tilde{\mathbf{s}}_1^{(m)}\}$; window size $r$; ridge coefficient $\lambda$
\Ensure Coefficient bank $\mathcal{C}=\{\mathbf{c}_{N-r+1},\mathbf{c}_{N-r},\ldots,\mathbf{c}_1\}$

\For{$n = N-r+1,N-r,\ldots,1$}
    \For{$m = 1,2,\ldots,M$}
        \For{$j = 1, \ldots, r-1$}
            \State $\Delta_{n,j}^{(m)} \gets \tilde{\mathbf{s}}_{n+j}^{(m)} - \tilde{\mathbf{s}}_{n+j-1}^{(m)}$
        \EndFor
        \State $\mathbf{B}_n^{(m)} \gets \left[ \Delta_{n,r-1}^{(m)}, \ldots, \Delta_{n,1}^{(m)}, \tilde{\mathbf{s}}_n^{(m)} \right]$
    \EndFor
    \State $\mathbf{c}_n^\star \gets \arg\min_{\mathbf{c}\in\mathbb{R}^r} \sum_{m=1}^M \left\| \tilde{\mathbf{s}}_{n-1}^{(m)} - \mathbf{B}_n^{(m)}\mathbf{c} \right\|_2^2 + \lambda \|\mathbf{c}\|_2^2$
\EndFor

\State \Return $\mathcal{C}=\{\mathbf{c}_{N-r+1}^\star, \mathbf{c}_{N-r}^\star, \ldots, \mathbf{c}_1^\star\}$
\end{algorithmic}
\end{algorithm}

\begin{algorithm}[t]
\caption{LEAP Sampling (Predict-Correct)}
\label{alg:prediction_correction}
\begin{algorithmic}[1]
\Require Initial state $\tilde{\mathbf{x}}_{N}$; coefficient bank $\mathcal{C}$; window size $r$; score network $s_{\theta}$
\Ensure Final sample $\tilde{\mathbf{x}}_0$

\State Warm-up: run base sampler for $r$ steps
\State $\mathcal{B} \gets \{\tilde{\mathbf{s}}_{N}, \tilde{\mathbf{s}}_{N-1}, \ldots, \tilde{\mathbf{s}}_{N-r+1}\}$
\State Update current state to $\tilde{\mathbf{x}}_{N-r}$

\For{$n = N-r+1$ \textbf{down to} $3$ \textbf{step} $2$}
    \State \textbf{Predict:}
    \State \quad Compute differences $\Delta_{n,j}$ from $\mathcal{B}$ for $j=1,\ldots,r-1$
    \State \quad $\hat{\mathbf{s}}_{n-1} \gets \mathbf{c}_{n,r-1}\tilde{\mathbf{s}}_{n} + \sum_{j=0}^{r-2} \mathbf{c}_{n,j}\,\Delta_{n,r-j-1}$
    \State \quad Evolve to $\hat{\mathbf{x}}_{n-2}$ via numerical solver using $\hat{\mathbf{s}}_{n-1}$
    \State \quad Append $\hat{\mathbf{s}}_{n-1}$ to $\mathcal{B}$ and remove the oldest score

    \State \textbf{Correct:}
    \State \quad Evaluate exact score $\tilde{\mathbf{s}}_{n-2} \gets s_{\theta}(\hat{\mathbf{x}}_{n-2})$
    \State \quad Evolve to $\tilde{\mathbf{x}}_{n-3}$ via numerical solver using $\tilde{\mathbf{s}}_{n-2}$
    \State \quad Append $\tilde{\mathbf{s}}_{n-2}$ to $\mathcal{B}$ and remove the oldest score
\EndFor

\State \Return $\tilde{\mathbf{x}}_0$
\end{algorithmic}
\end{algorithm}

While the formulation above is conceptually straightforward, directly regressing on $[\tilde{\mathbf{s}}_{n+r-1},\tilde{\mathbf{s}}_{n+r-2},\ldots,\tilde{\mathbf{s}}_{n}]$ is often numerically ill-posed. Since consecutive model outputs along smooth segments of the diffusion path are highly correlated, the columns of the history matrix tend to be nearly collinear. As a result, the associated least-squares system becomes poorly conditioned, making the fitted coefficients sensitive to finite-sample error, imperfect cross-trajectory alignment, and numerical instabilities. To mitigate this issue, we reparameterize the local subspace using a \emph{difference basis} rather than the \emph{raw basis}. Specifically, for each trajectory window, we define the first-order differences of the scores as:
\begin{equation}
\Delta_{n,j}^{(m)}:=\tilde{\mathbf{s}}_{n+j}^{(m)}-\tilde{\mathbf{s}}_{n+j-1}^{(m)},\qquad j=1,\ldots,r-1,
\end{equation}
and define the modified history matrix as:
\begin{equation}
\mathbf{B}_n^{(m)}=
\left[
\Delta_{n,r-1}^{(m)},
\ldots,
\Delta_{n,1}^{(m)} ,\tilde{\mathbf{s}}_n^{(m)}\right]\in\mathbb{R}^{d\times r}.
\end{equation}
Rather than directly estimating $\mathbf{w}_n$, we fit an auxiliary coefficient vector $\mathbf{c}_n\in\mathbb{R}^r$ such that
\begin{equation}
\tilde{\mathbf{s}}_{n-1}^{(m)} \approx \mathbf{B}_n^{(m)}\mathbf{c}_n.
\end{equation}
In practice, we aggregate samples from all reference trajectories and solve a ridge regression problem within the difference basis for each timestep $n=N-1,N-2,\ldots,r$:
\begin{equation}
\mathbf{c}_n^\star=\arg\min_{\mathbf{c}\in\mathbb{R}^r}\sum_{m=1}^M\left\|\tilde{\mathbf{s}}_{n-1}^{(m)}-\mathbf{B}_n^{(m)}\mathbf{c}\right\|_2^2+\lambda\|\mathbf{c}\|_2^2.
\end{equation}

Once the step-wise coefficient bank $\mathcal{C}=\{\mathbf{c}_{N-r+1},\mathbf{c}_{N-r},\ldots,\mathbf{c}_{1}\}$ has been solved, inference proceeds in two phases. First, we execute a warm-up phase by running the initial $r$ denoising steps in a standard way. This initializes the history buffer with the initial $r$ model outputs and their corresponding states:
\begin{equation}
\{\tilde{\mathbf{s}}_{N},\tilde{\mathbf{s}}_{N-2},\ldots,\tilde{\mathbf{s}}_{N-r+1}\}, \quad \{\tilde{\mathbf{x}}_{N},\tilde{\mathbf{x}}_{N-2},\ldots,\tilde{\mathbf{x}}_{N-r}\}.
\end{equation}
This warm-up phase provides the minimum context required by the extrapolator. Moving forward, we employ an interleaved sampling strategy that alternates between a network-free prediction step and a network-based correction step. Specifically, suppose the current history buffer contains $r$ most recent scores:
\begin{equation}
\{\tilde{\mathbf{s}}_{n+r-1},\tilde{\mathbf{s}}_{n+r-2},\ldots,\tilde{\mathbf{s}}_{n}\}.
\end{equation}
We first predict the next model output using the learned step-wise coefficients:
\begin{equation}
\hat{\mathbf{s}}_{n-1}
:=
\mathbf{c}_{n,r-1}\tilde{\mathbf{s}}_{n}+\sum_{j=0}^{r-2} \mathbf{c}_{n,j}\,\Delta_{n,r-j-1}.
\end{equation}
The extrapolated score $\hat{\mathbf{s}}_{n-1}$ is then utilized to evolve the trajectory to $\hat{\mathbf{x}}_{n-2}$ via numerical solvers, eliminating the need for an explicit network evaluation. The extrapolated score $\hat{\mathbf{s}}_{n-1}$ then serves as a proxy for $\tilde{\mathbf{s}}_{n-1}$ and is pushed into the history buffer, displacing the oldest entry to maintain a fixed window size of $r$.

Extrapolation alone inevitably accumulates numerical drift over time. To ensure trajectory stability, we implement an interleaved Predict-Correct scheme. Following the short warm-up stage, the solver alternates between network-free extrapolation and exact network evaluation, balancing computational efficiency with numerical robustness, providing a practical acceleration framework for diffusion sampling. Given the step-wise optimality of our linear extrapolators and their capability to bypass redundant network evaluations, we name our method as \textit{Optimal Linear Leaps} (LEAP). The full algorithms of coefficient bank construction and LEAP sampling are provided in Algorithm \ref{alg:bank_construction} and \ref{alg:prediction_correction}.

\renewcommand{\arraystretch}{0.93}
\begin{table*}[t]
\centering
\caption{Quantitative results on SD3.5-Medium and FLUX.1-dev. The best and second-best results are marked in bold and \underline{underlined}, respectively. Our proposed method (LEAP) is highlighted with a gray background.}
\label{tab:main_results}
\resizebox{\textwidth}{!}{
\begin{tabular}{llccc|ccc}
\toprule
\multirow{2}{*}{\textbf{Steps}} & \multirow{2}{*}{\textbf{Method}} & \multicolumn{3}{c|}{\textbf{SD3.5-Medium}} & \multicolumn{3}{c}{\textbf{FLUX.1-dev}} \\
\cmidrule(lr){3-5} \cmidrule(lr){6-8}
& & \textbf{HPSv2} $\uparrow$ & \textbf{ImageReward} $\uparrow$ & \textbf{CLIPScore} $\uparrow$ & \textbf{HPSv2} $\uparrow$ & \textbf{ImageReward} $\uparrow$ & \textbf{CLIPScore} $\uparrow$ \\
\midrule

\multirow{8}{*}{10} 
& Euler & 25.8866 & 0.6653 & 32.7920 & 28.3538 & 0.8110 & 30.6939 \\
& DPM++ & 26.4565 & 0.7178 & 32.9027 & 28.6365 & 0.8331 & 30.7257 \\
& iPNDM & 26.4041 & 0.7063 & 32.7579 & 28.8272 & 0.8956 & 30.9178 \\
& UniPC & 26.4450 & 0.7122 & 32.6428 & 28.8193 & 0.8839 & 31.0634 \\
& LTC-Accel & 26.2546 & 0.7324 & 33.1578 & 28.9543 & 0.9193 & 31.1217 \\
& STORK-4 & 26.0567 & 0.6619 & 32.8704 & \underline{29.2846} & \underline{0.9729} & \underline{31.2895} \\
\cmidrule{2-8}
\rowcolor{gray!15} & \textbf{LEAP} (Euler) & \underline{26.8419} & \underline{0.7692} & \textbf{33.2982} & 29.0536 & 0.9204 & 30.9412 \\
\rowcolor{gray!15} & \textbf{LEAP} (DPM++) & \textbf{27.0237} & \textbf{0.8177} & \underline{33.2850} & \textbf{29.4004} & \textbf{0.9816} & \textbf{31.4320} \\
\midrule

\multirow{8}{*}{15} 
& Euler & 27.4725 & 0.8416 & 33.4644 & 29.2580 & 0.9470 & 31.0488 \\
& DPM++ & 27.8642 & 0.8438 & 33.4704 & 29.5030 & 0.9588 & 31.2782 \\
& iPNDM & 27.9220 & 0.8770 & 33.4416 & 29.5056 & 0.9491 & \underline{31.3836} \\
& UniPC & 27.8680 & 0.8788 & 33.4703 & 29.4668 & 0.9414 & 31.2339 \\
& LTC-Accel & 27.8174 & 0.9196 & \textbf{33.7184} & 29.5874 & 0.9576 & 31.1906 \\
& STORK-4 & 27.4224 & 0.8253 & 33.5179 & \textbf{29.9693} & \textbf{1.0308} & 31.3041 \\
\cmidrule{2-8}
\rowcolor{gray!15} & \textbf{LEAP} (Euler) & \underline{28.1826} & \underline{0.9522} & \underline{33.6817} & 29.6785 & 0.9719 & 31.1046 \\
\rowcolor{gray!15} & \textbf{LEAP} (DPM++) & \textbf{28.2152} & \textbf{0.9806} & 33.5617 & \underline{29.9024} & \underline{1.0101} & \textbf{31.4981} \\
\midrule

\multirow{8}{*}{20} 
& Euler & 28.2691 & 0.9201 & 33.6945 & 29.4556 & 0.9398 & 30.9759 \\
& DPM++ & 28.4878 & 0.9560 & 33.7391 & 29.6523 & 0.9460 & 31.0277 \\
& iPNDM & 28.4252 & 0.9317 & \textbf{33.9325} & 29.6963 & 0.9886 & 31.2203 \\
& UniPC & 28.4756 & 0.9364 & 33.7757 & 29.6773 & 0.9840 & 31.1858 \\
& LTC-Accel & 28.3403 & 0.9547 & 33.7901 & 29.7326 & 0.9884 & 31.1142 \\
& STORK-4 & 28.1835 & 0.9035 & \underline{33.8360} & 29.7927 & 1.0066 & 31.2650 \\
\cmidrule{2-8}
\rowcolor{gray!15} & \textbf{LEAP} (Euler) & \underline{28.5727} & \underline{0.9701} & 33.7192 & \underline{29.9370} & \underline{1.0320} & \underline{31.3433} \\
\rowcolor{gray!15} & \textbf{LEAP} (DPM++) & \textbf{28.6259} & \textbf{0.9915} & 33.5356 & \textbf{30.1514} & \textbf{1.0514} & \textbf{31.6484} \\
\bottomrule
\end{tabular}
}
\end{table*}

\section{Experiments}
\label{sec:experiments}
We evaluate the proposed LEAP on three cutting-edge diffusion models: SD3.5-Medium~\cite{esser2024scaling} and FLUX.1-Dev~\cite{flux1dev} for text-image task and TRELLIS.2~\cite{xiang2025native} for image-3D task. The quality of generated images is quantitatively evaluated using three metrics, including ImageReward~\cite{xu2023imagereward}, CLIP Score~\cite{radford2021learning}, and HPSv2~\cite{wu2023human}. The prompts are drawn from the DrawBench~\cite{saharia2022photorealistic} dataset. For image-3D task, we evaluate 3D object on a test set of 100 randomly sampled images from the Google Scanned Objects (GSO) dataset~\cite{downs2022google} by PSNR and CLIP Score. By default, we fix a window size of $r=5$ and set the ridge regularization coefficient to $\lambda=0.0001$ across all experiments. Unless otherwise specified, LEAP is compared with existing state-of-the-art solvers under identical inference settings. 

\subsection{Main Results}
\textit{\textbf{Text-to-image Generation}}. 
We evaluate LEAP on text-to-image generation, benchmarking it against representative solvers including Euler method~\cite{song2020denoising}, DPM++~\cite{lu2025dpm}, UniPC~\cite{zhao2023unipc}, iPNDM~\cite{von-platen-etal-2022-diffusers}, LTC-Accel~\cite{zhu2025accelerating} and STORK-4~\cite{tan2025stork}. As shown in Tab.~\ref{tab:main_results}, LEAP consistently achieves top-tier performance across diverse sampling budgets, effectively preserving both text-image alignment and perceptual quality for accelerated sampling. 
Moreover, LEAP consistently enhances the performance of both first-order and higher-order solvers, demonstrating its versatile plug-and-play capability. By leveraging a pre-computed, step-wise coefficient bank, LEAP seamlessly adapts any solver without any solver-specific hyperparameter tuning or model fine-tuning, making it a solver-agnostic enhancement module.
Fig.~\ref{fig:qualitative_comparison_1} provides a qualitative comparison of LEAP against state-of-the-art solvers on SD3.5-Medium. Our LEAP exhibits superior performance across four challenging scenarios including (1) text rendering: LEAP accurately synthesizes  legible text, whereas baselines suffer from typographic errors; (2) physical lighting: LEAP captures realistic transparency and intricate light refraction; (3) attribute binding: LEAP precisely binds the localized attributes to their respective partitions without semantic leakage.

\renewcommand{\arraystretch}{0.9}
\begin{table*}[t]
  \centering
  \caption{Quantitative comparison of image-to-3D generation on the GSO dataset via TRELLIS.2.}
  \label{tab:main_results_3d}
  \resizebox{0.9\linewidth}{!}{
  \begin{tabular}{l ccccc ccccc}
    \toprule
    \multirow{2}{*}{\textbf{Method}} & \multicolumn{5}{c}{\textbf{PSNR $\uparrow$}} & \multicolumn{5}{c}{\textbf{CLIP Score $\uparrow$}} \\
    \cmidrule(lr){2-6} \cmidrule(lr){7-11}
    & 10 & 15 & 20 & 25 & 30 & 10 & 15 & 20 & 25 & 30 \\
    \midrule
    Euler & 7.250 & 7.396 & 7.235 & 7.361 & 7.405 & 0.8638 & 0.8654 & 0.8682 & 0.8650 & 0.8658 \\
    \rowcolor{gray!15} \textbf{LEAP} (Euler) & \textbf{7.414} & \textbf{7.537} & \textbf{7.337} & \textbf{7.410} & \textbf{7.408} & \textbf{0.8677} & \textbf{0.8674} & \textbf{0.8693} & \textbf{0.8678} & \textbf{0.8709} \\
    \bottomrule
  \end{tabular}
  }
\end{table*}

\textit{\textbf{Image-to-3D Generation}}.
We further validate LEAP on image-to-3D generation using the TRELLIS.2~\cite{xiang2025native} architecture. 
As presented in Tab.~\ref{tab:main_results_3d}, our method achieves systematic improvements in both PSNR and CLIP Score over the baseline on GSO dataset.
The qualitative results shown in Fig.~\ref{fig:qualitative_comparison_2} demonstrate LEAP's superior capability in (1) topological correctness, maintaining precise connectivity in complex jointed structures; (2) surface fidelity; (3) geometric completeness, effectively generating intricate and high-frequency geometries, yielding sharper textural details and superior color accuracy.

\begin{figure}[t]
    \centering
    \setlength{\tabcolsep}{1pt} 
    \renewcommand{\arraystretch}{0.5} 

    \begin{tabular}{m{1.3cm} m{2.3cm} m{2.3cm} m{2.3cm}}

        \makecell[c]{\small DPM++} & 
        \includegraphics[width=\linewidth]{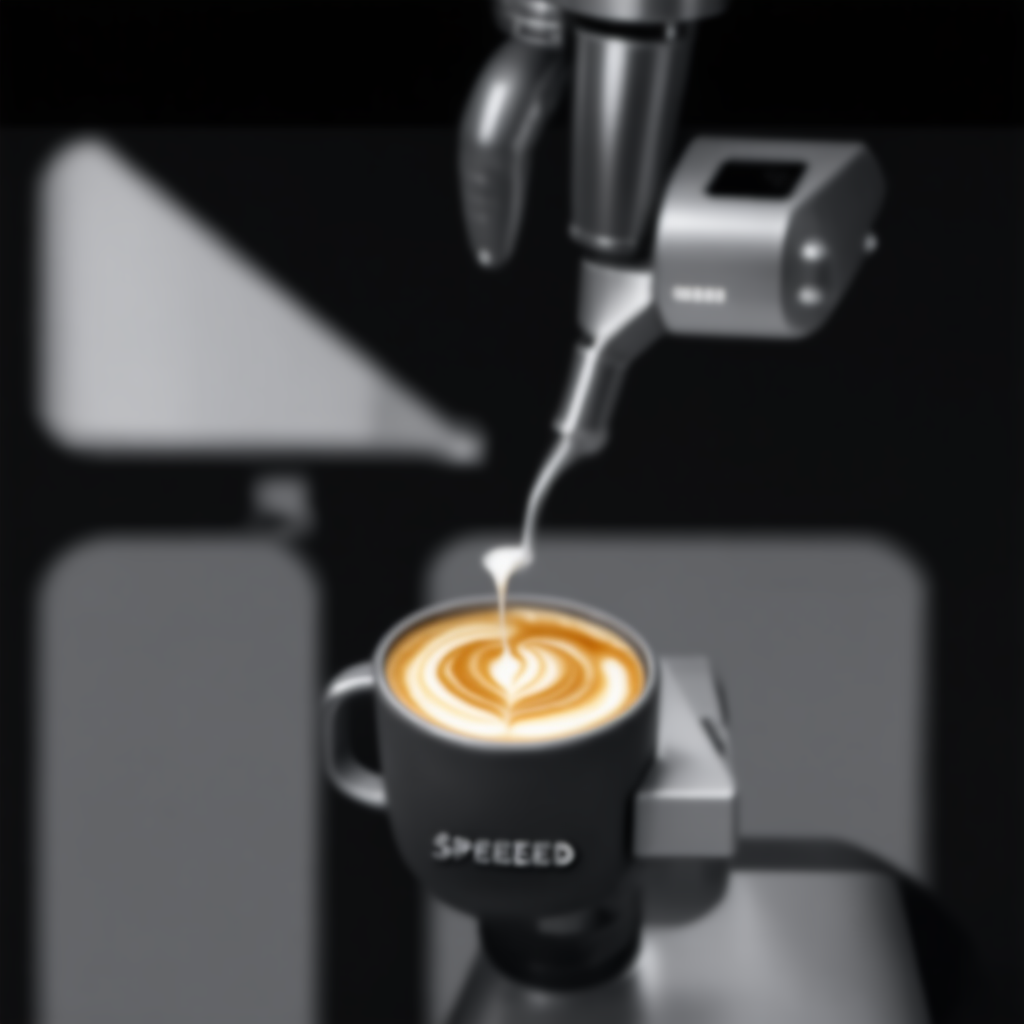} & 
        \includegraphics[width=\linewidth]{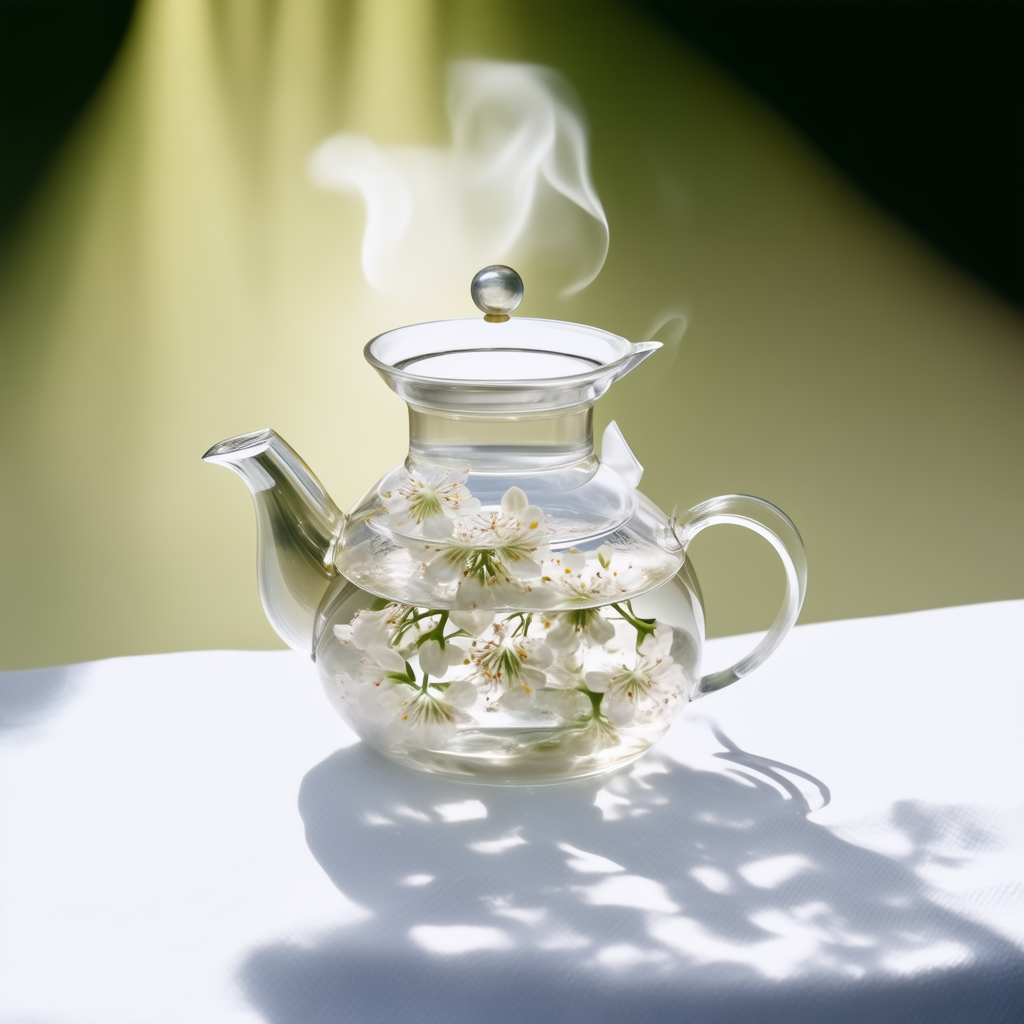} & 
        \includegraphics[width=\linewidth]{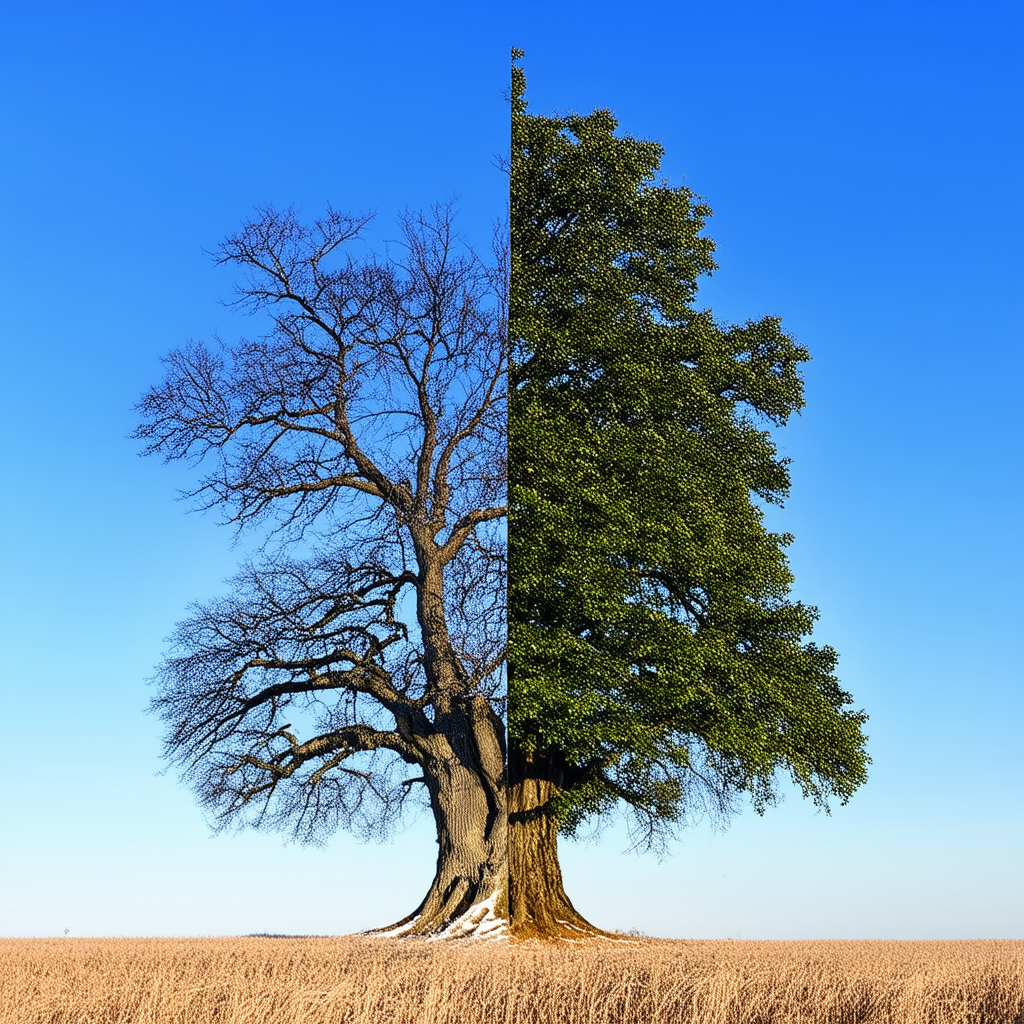} \\
        \addlinespace[0.2mm] 

        \makecell[c]{\small LTC-Accel} & 
        \includegraphics[width=\linewidth]{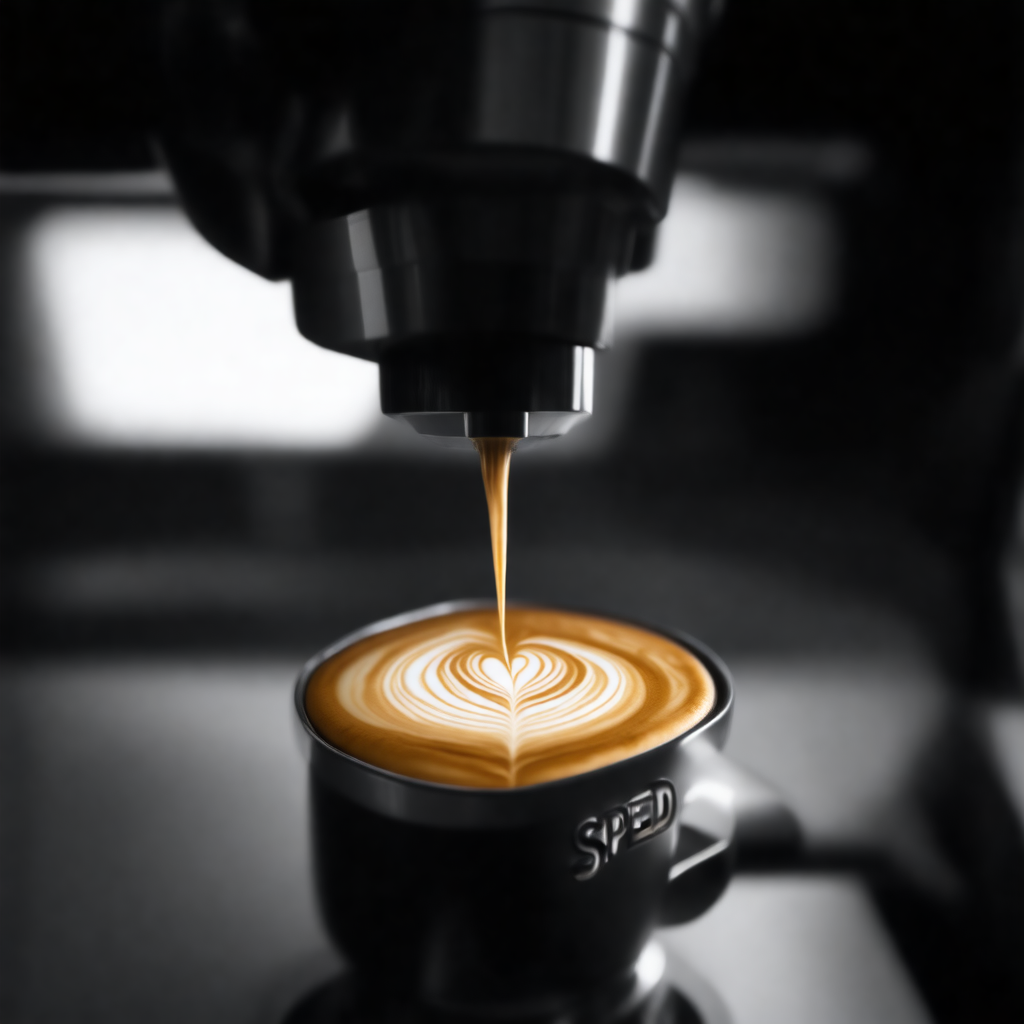} & 
        \includegraphics[width=\linewidth]{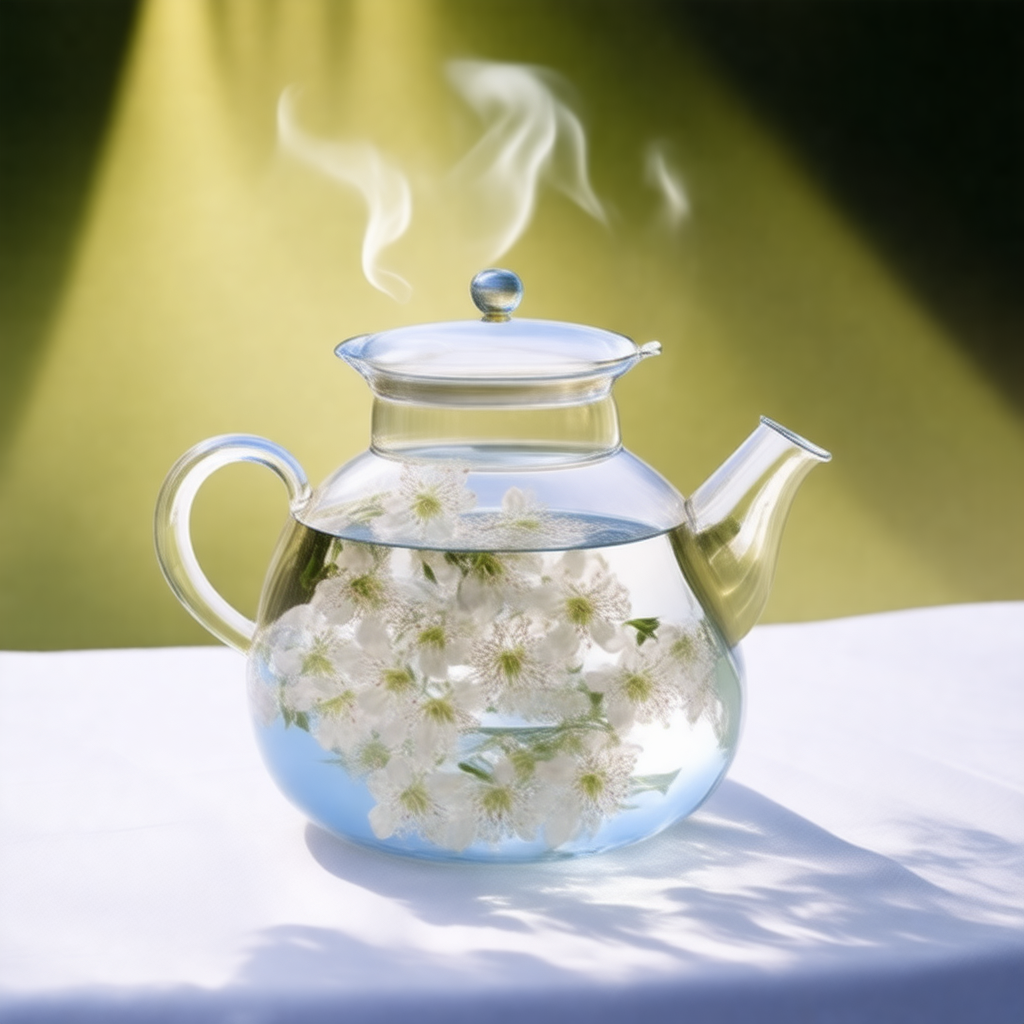} & 
        \includegraphics[width=\linewidth]{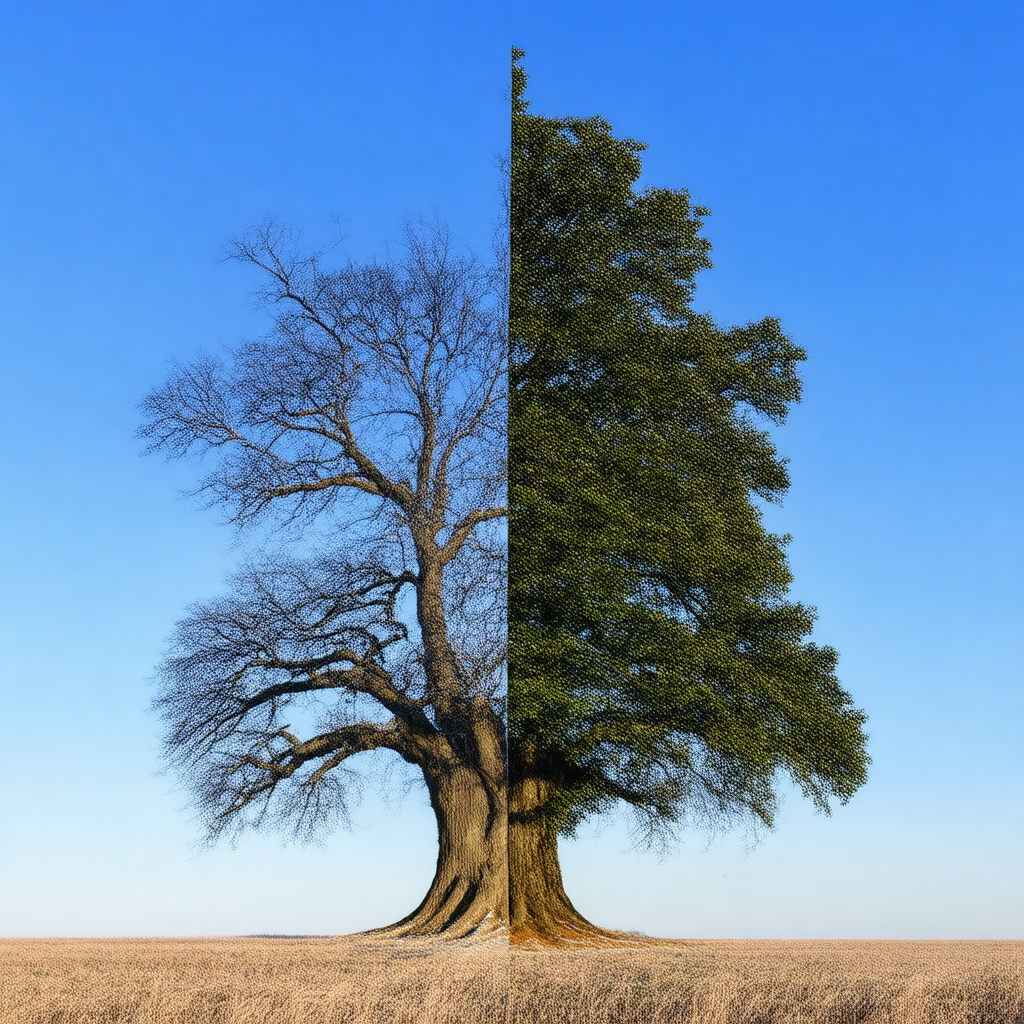} \\
        \addlinespace[0.2mm]

        \makecell[c]{\small STORK4} & 
        \includegraphics[width=\linewidth]{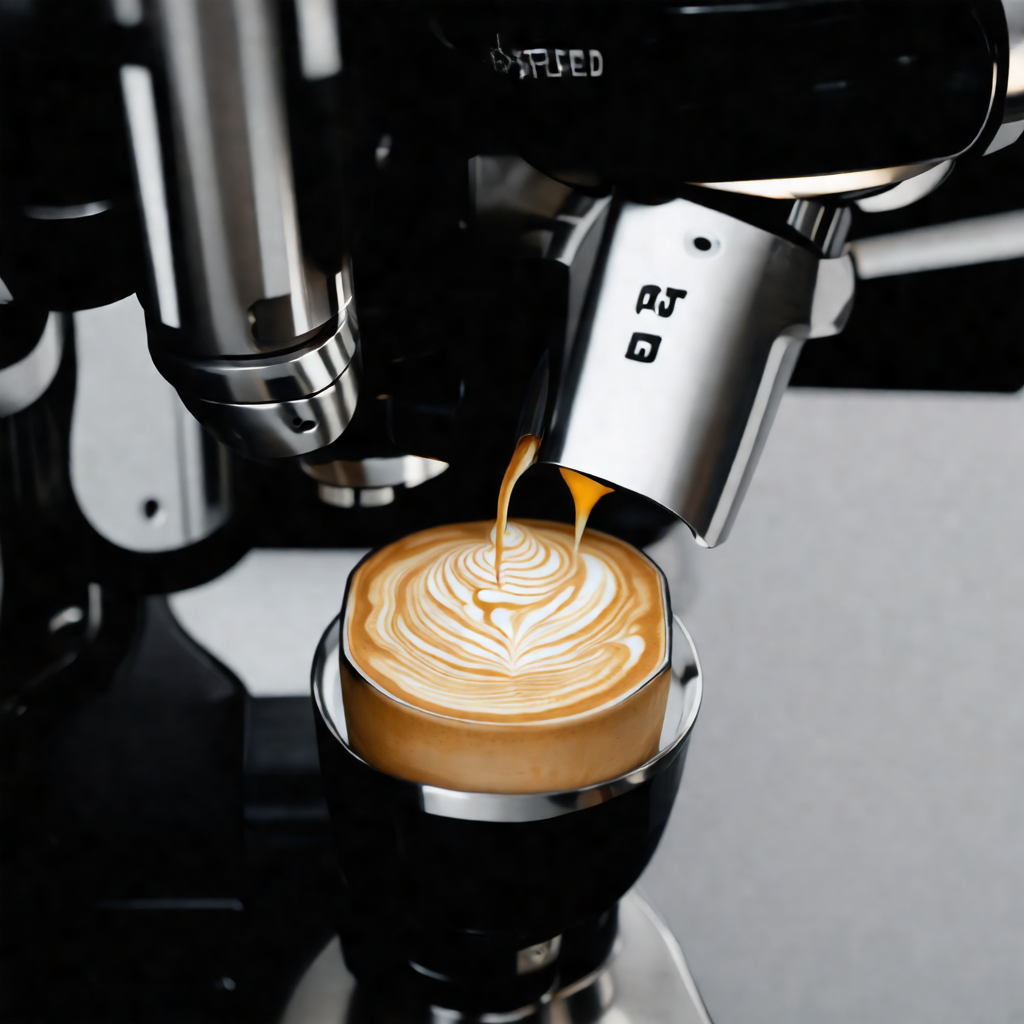} & 
        \includegraphics[width=\linewidth]{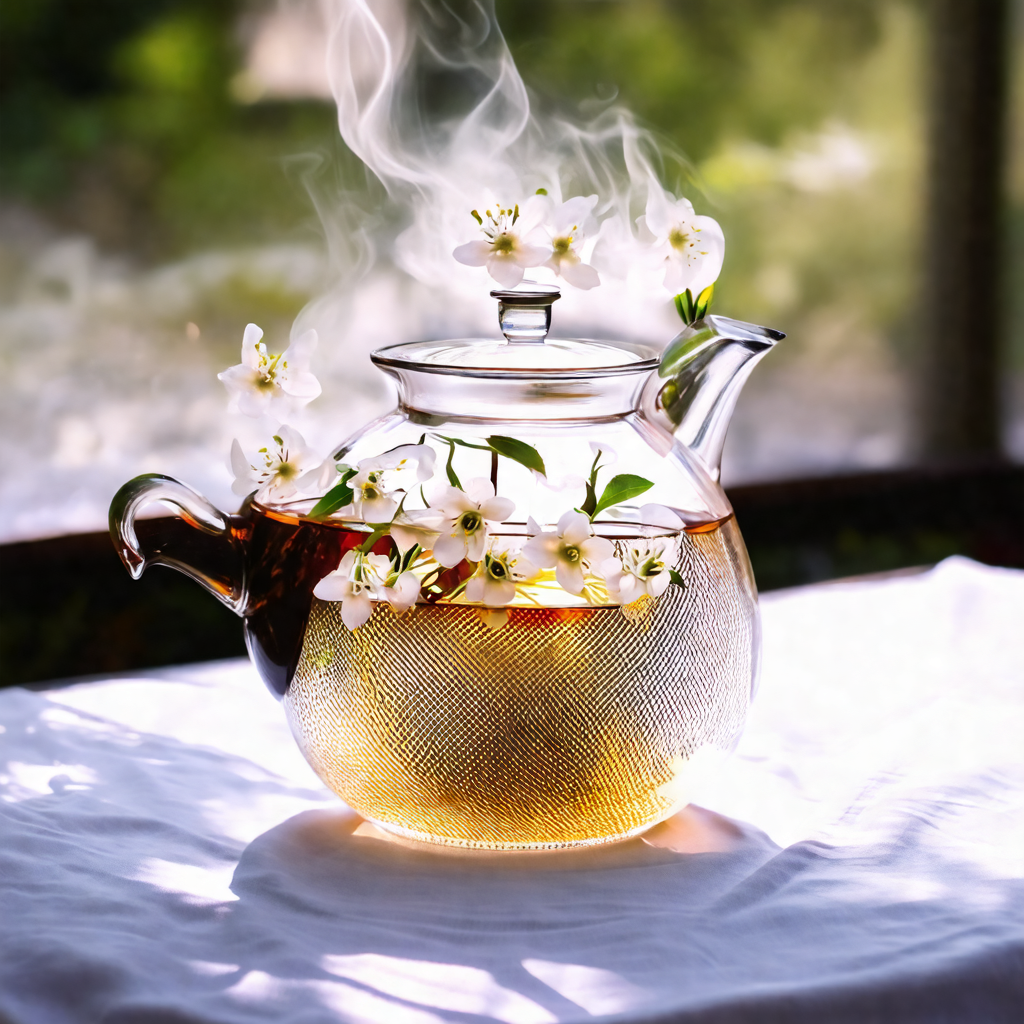} & 
        \includegraphics[width=\linewidth]{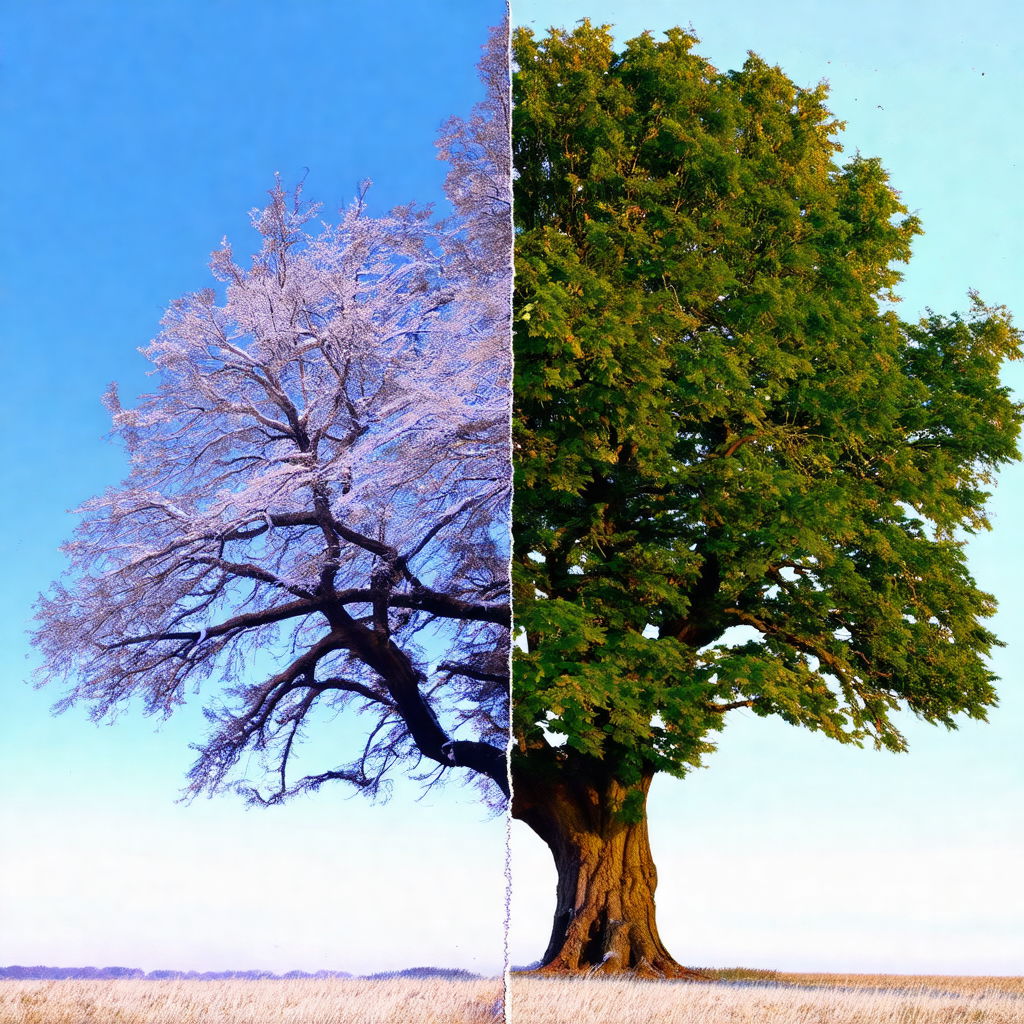} \\
        \addlinespace[0.2mm]
        
        \makecell[c]{\small \textbf{LEAP}} & 
        \includegraphics[width=\linewidth]{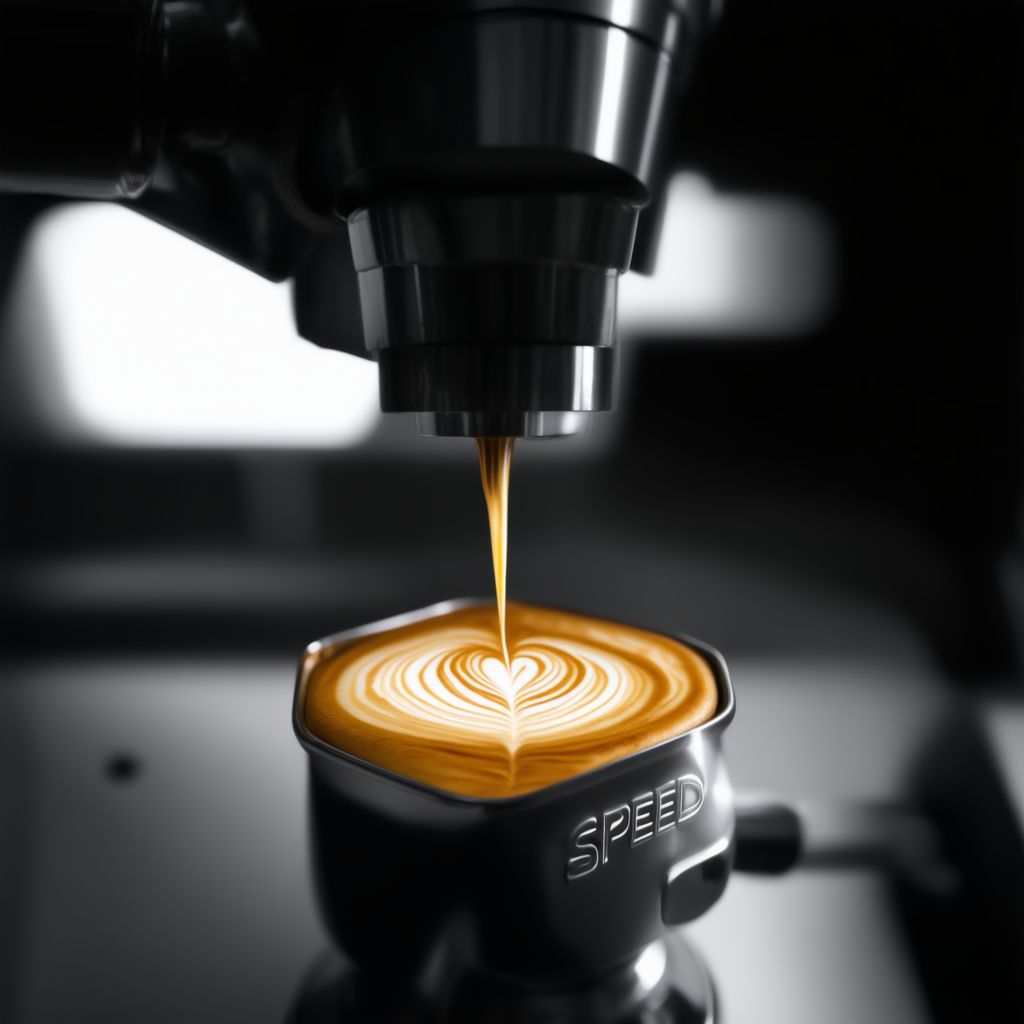} & 
        \includegraphics[width=\linewidth]{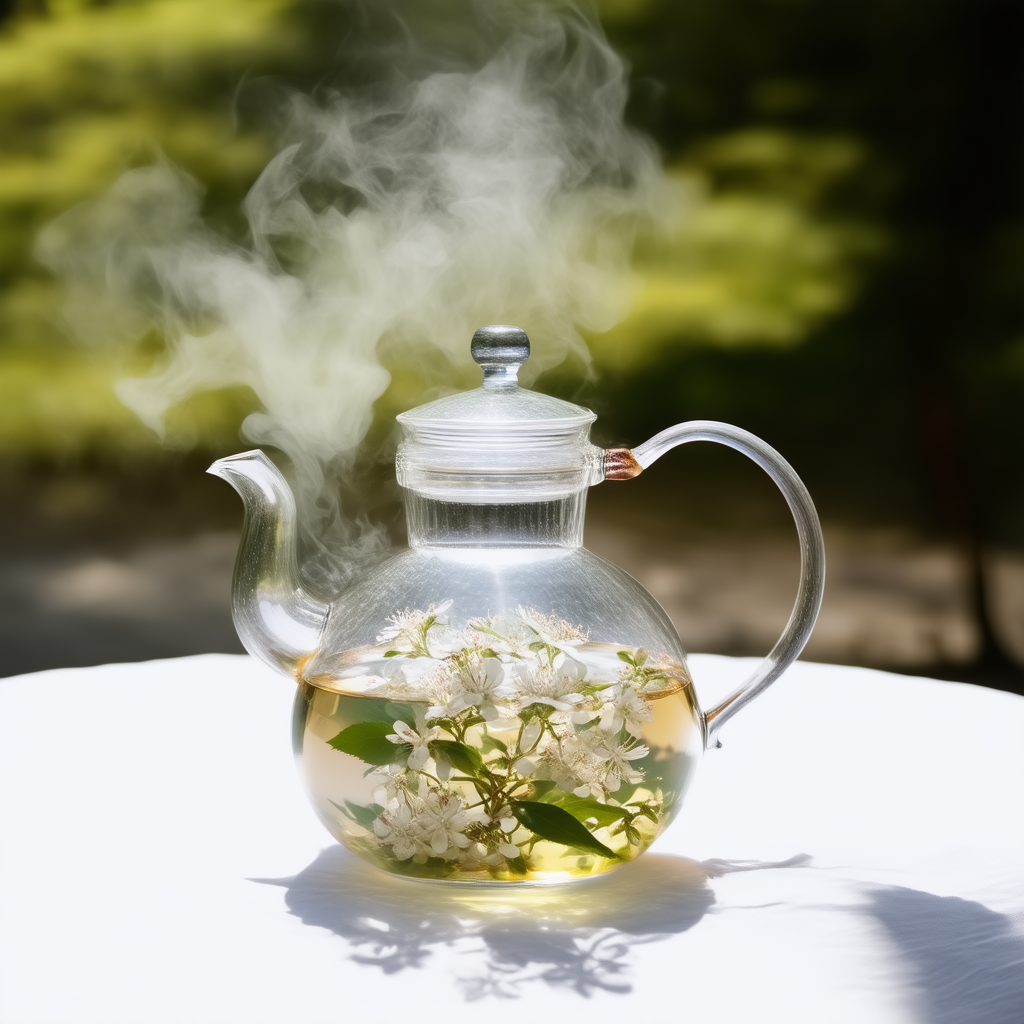} & 
        \includegraphics[width=\linewidth]{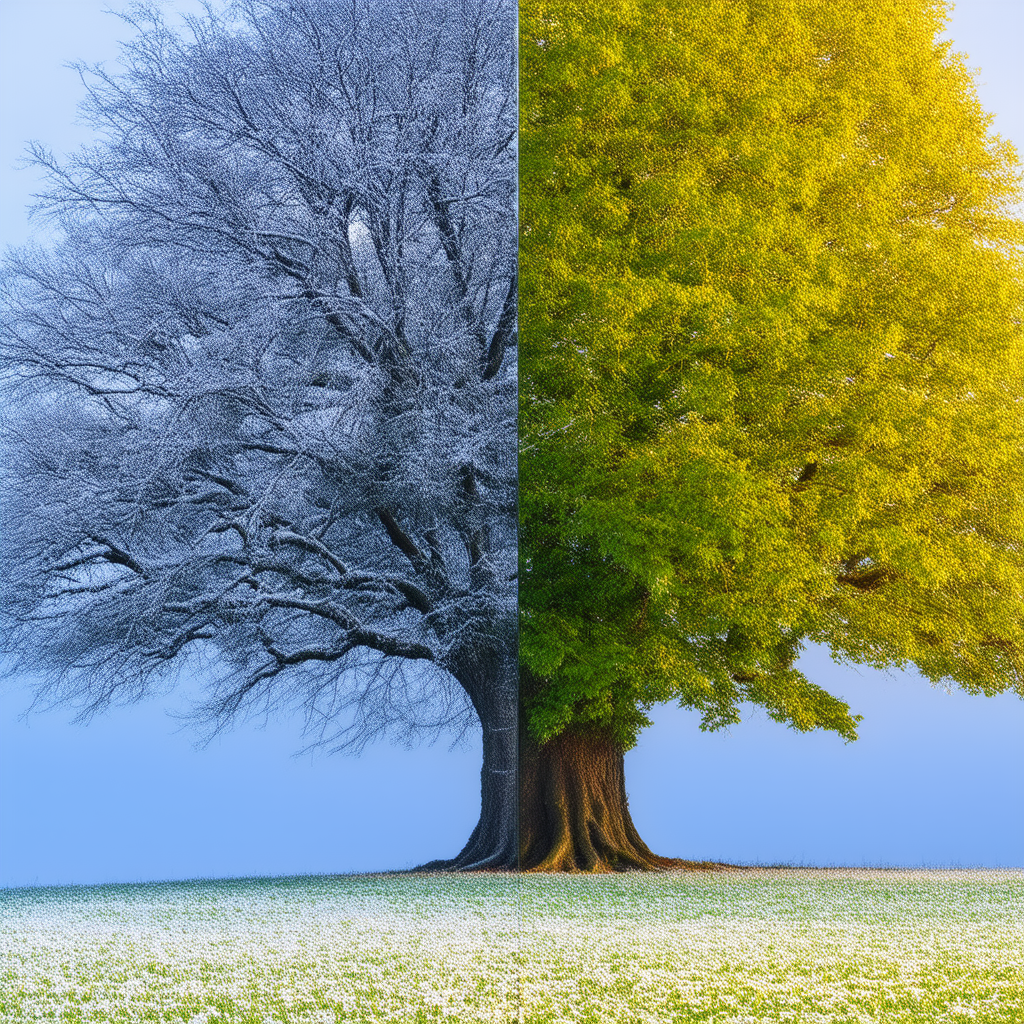} \\

    \end{tabular}
    
    \caption{Text-to-image generation on SD3.5-Medium using 10 inference steps across diverse prompts.}
    \label{fig:qualitative_comparison_1}
\end{figure}

\begin{figure}[t]
    \centering
    \setlength{\tabcolsep}{2.1pt}
    \renewcommand{\arraystretch}{0.5}

    \setlength{\fboxsep}{0pt}
    \setlength{\fboxrule}{0.5pt}

    \begin{tabular}{ m{0.48\linewidth} m{0.48\linewidth} }
        
        \makecell[c]{Euler} & \makecell[c]{\textbf{LEAP}} \\
        \addlinespace[0.5mm]
        
        \fbox{\includegraphics[width=\linewidth]{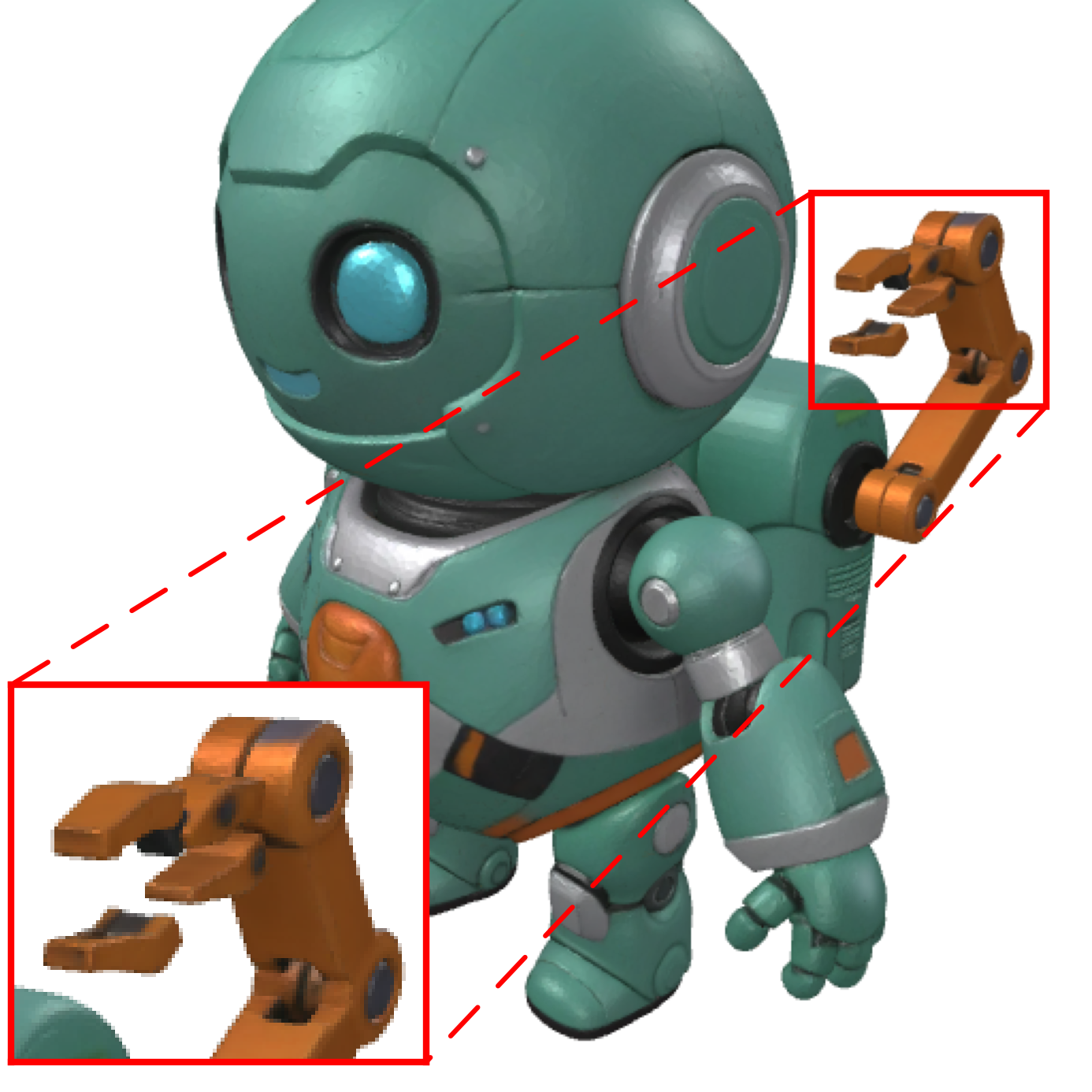}} & 
        \fbox{\includegraphics[width=\linewidth]{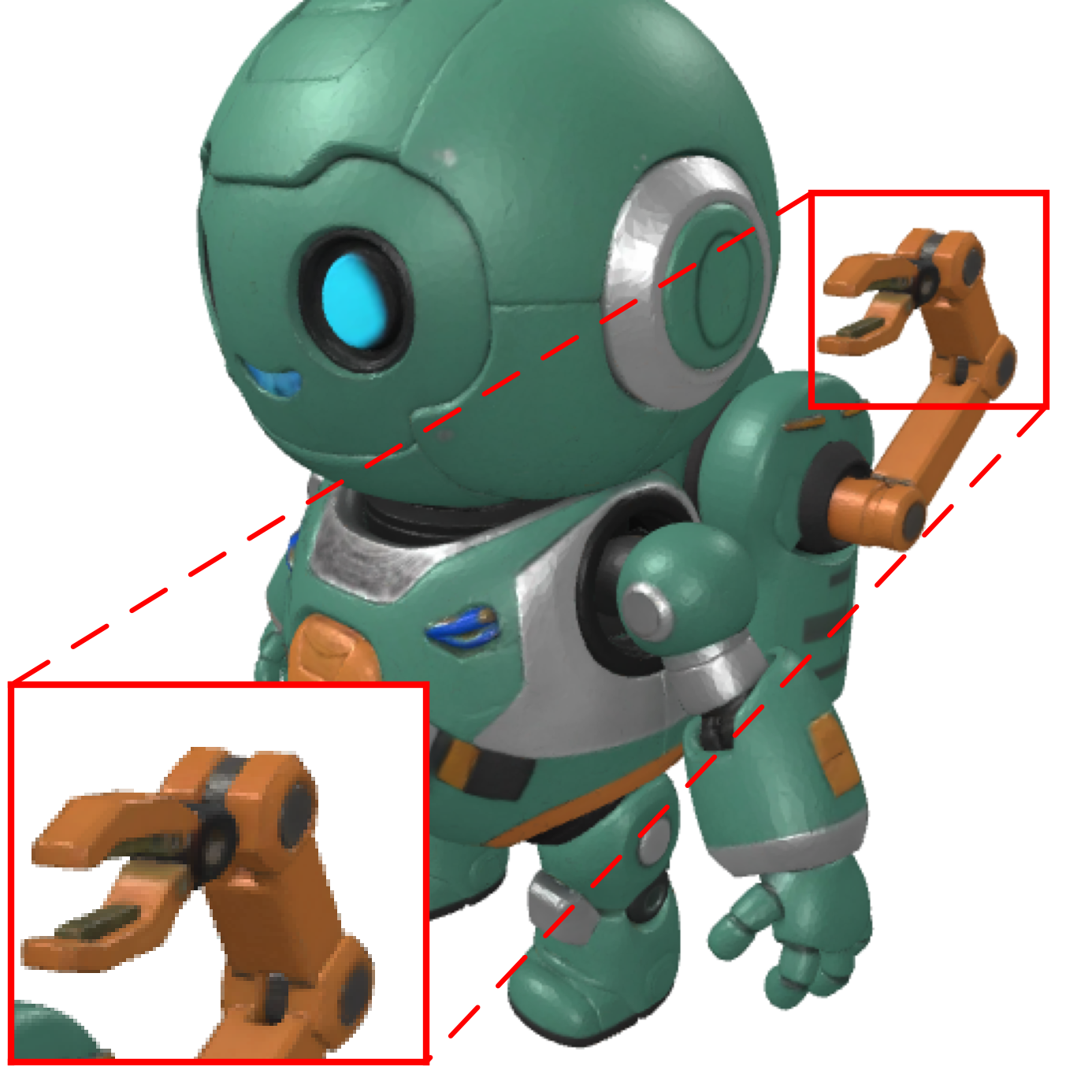}} \\
        \addlinespace[0.5mm] 

    \end{tabular}
    
    \begin{tabular}{m{0.232\linewidth} m{0.232\linewidth} m{0.232\linewidth} m{0.232\linewidth}}

        \fbox{\includegraphics[width=\linewidth]{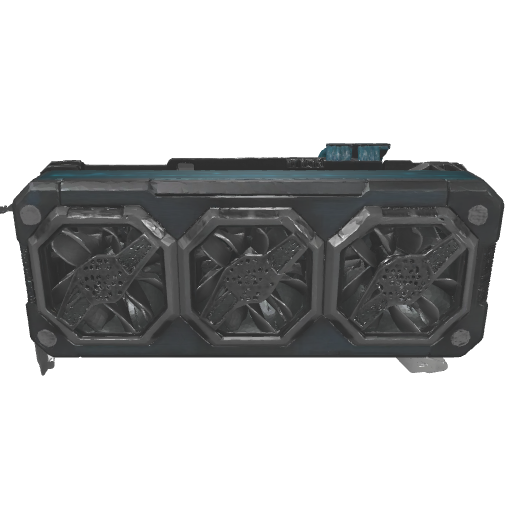}} & 
        \fbox{\includegraphics[width=\linewidth]{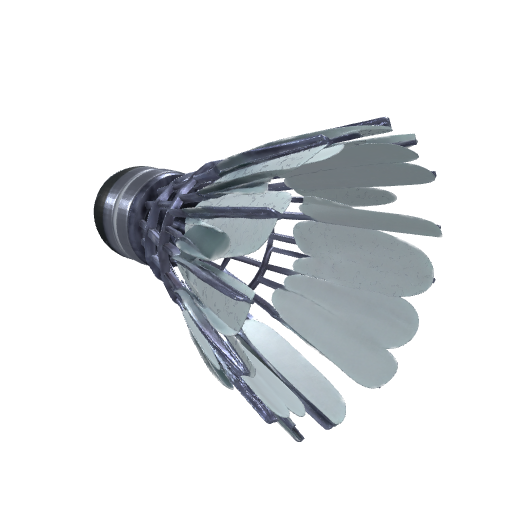}} & 
        \fbox{\includegraphics[width=\linewidth]{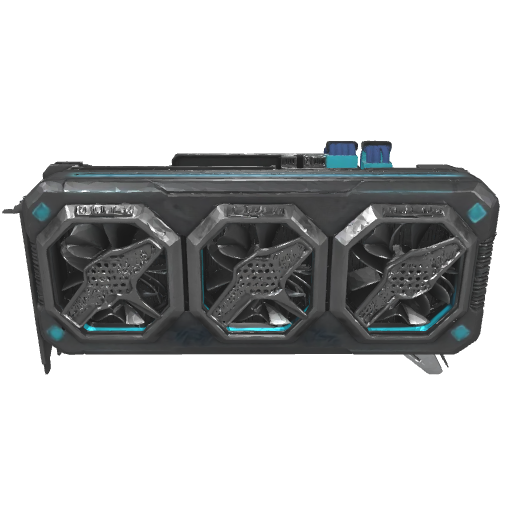}} &
        \fbox{\includegraphics[width=\linewidth]{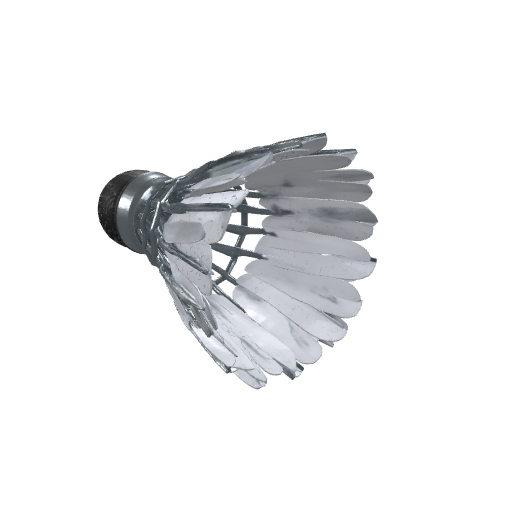}} \\
    \end{tabular}

    \caption{Image-to-3D generation on TRELLIS.2 using 10 inference steps across different image prompts.}
    \label{fig:qualitative_comparison_2}
    \vspace{-10pt}
\end{figure}

\subsection{Ablation Study}
To systematically validate the contribution of each key design choice in the proposed method, we conduct comprehensive ablation studies on SD3.5-Medium. 
To ensure a rigorous comparison, we isolate the impact of each component by maintaining a fixed $N$-step sampling schedule, identical random seeds and a consistent ensemble of reference trajectory (default $M=32$). 
(a) \textbf{Window Size} $r$. As shown in Fig~\ref{fig:ablation_ws}, we adopt $r=5$ as the default window size, as it delivers peak or near-peak performance in most scenarios. 
(b) \textbf{Penalty Strength} $\lambda$. Fig.~\ref{fig:ablation_lambda} demonstrates that LEAP is robust to the penalty coefficient $\lambda$ across a wide range of values. We retain this penalty term across all experiments to ensure theoretical consistency and numerical stability. 
(c) \textbf{Number of Reference Trajectories} $M$. As shown in Fig~\ref{fig:ablation_m}, LEAP exhibits remarkable data efficiency, maintaining high performance even with minimal reference data ($M \le 2$). This observation perfectly corroborates our theoretical analysis regarding cross-trajectory consistency. We default to $M=32$ to ensure statistically robust coverage across diverse prompts. 
(d) \textbf{Basis Construction.} Fig.~\ref{fig:basis_ablation} underscores the critical role of our difference-basis reparameterization. The difference basis effectively rectifies the ill-conditioning, ensuring numerical stability and enabling consistent quality improvements across all NFE settings.

\begin{figure*}[t]
    \centering
    \begin{subfigure}[b]{0.32\textwidth}
        \centering
        \includegraphics[width=\linewidth]{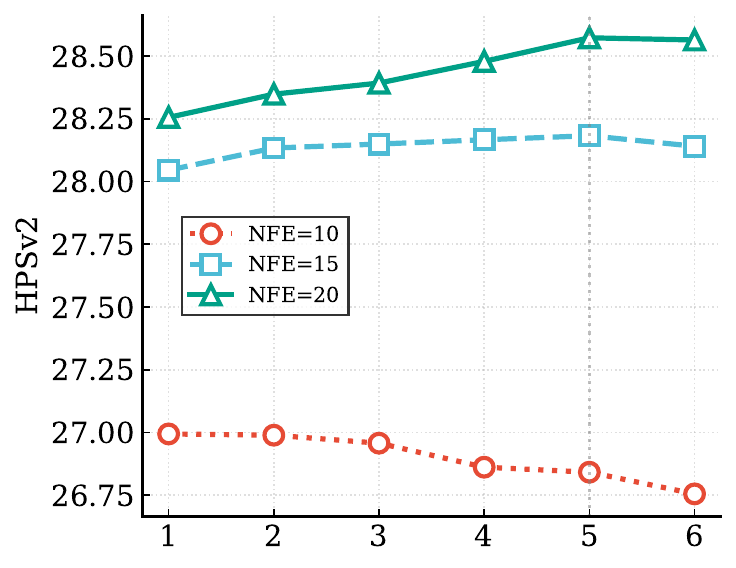}
        \caption{Window Size ($r$)}
        \label{fig:ablation_ws}
    \end{subfigure}
    \hfill
    \begin{subfigure}[b]{0.32\textwidth}
        \centering
        \includegraphics[width=\linewidth]{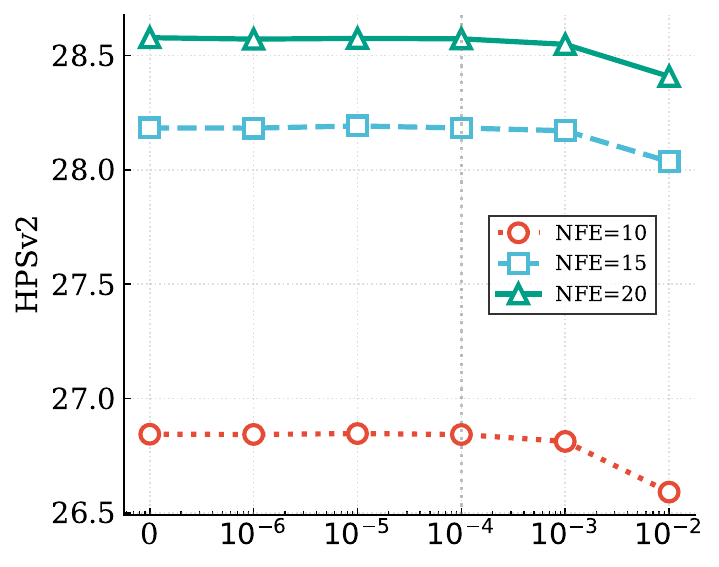}
        \caption{Penalty Strength ($\lambda$)}
        \label{fig:ablation_lambda}
    \end{subfigure}
    \hfill
    \begin{subfigure}[b]{0.32\textwidth}
        \centering
        \includegraphics[width=\linewidth]{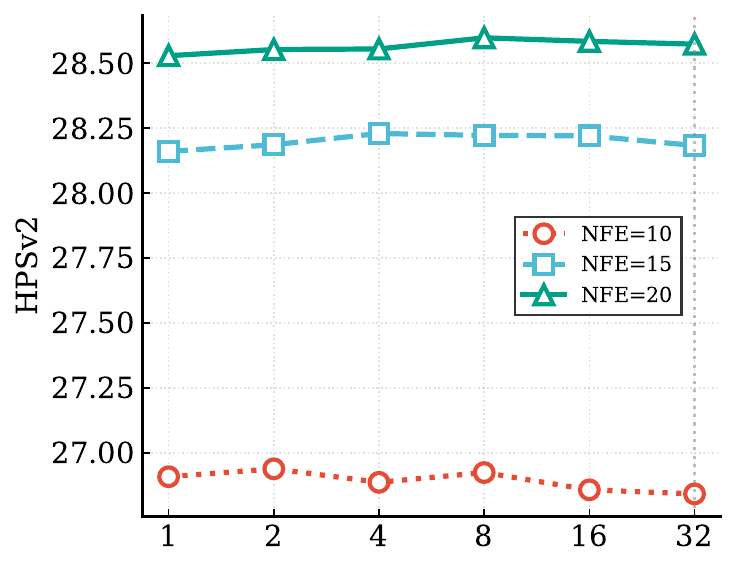}
        \caption{Number of Reference Trajectories ($M$)}
        \label{fig:ablation_m}
    \end{subfigure}
    
    \caption{Ablation studies on key hyperparameters in Stable Diffusion 3.5-Medium. We evaluate the impact of (a) Window Size, (b) Penalty Strength, and (c) Number of Tokens on the HPSv2 score across different NFEs. The default configuration is indicated by the dashed vertical line.}
    \label{fig:total_ablation}
\end{figure*}

\begin{figure}
    \centering
    \includegraphics[width=0.95\linewidth]{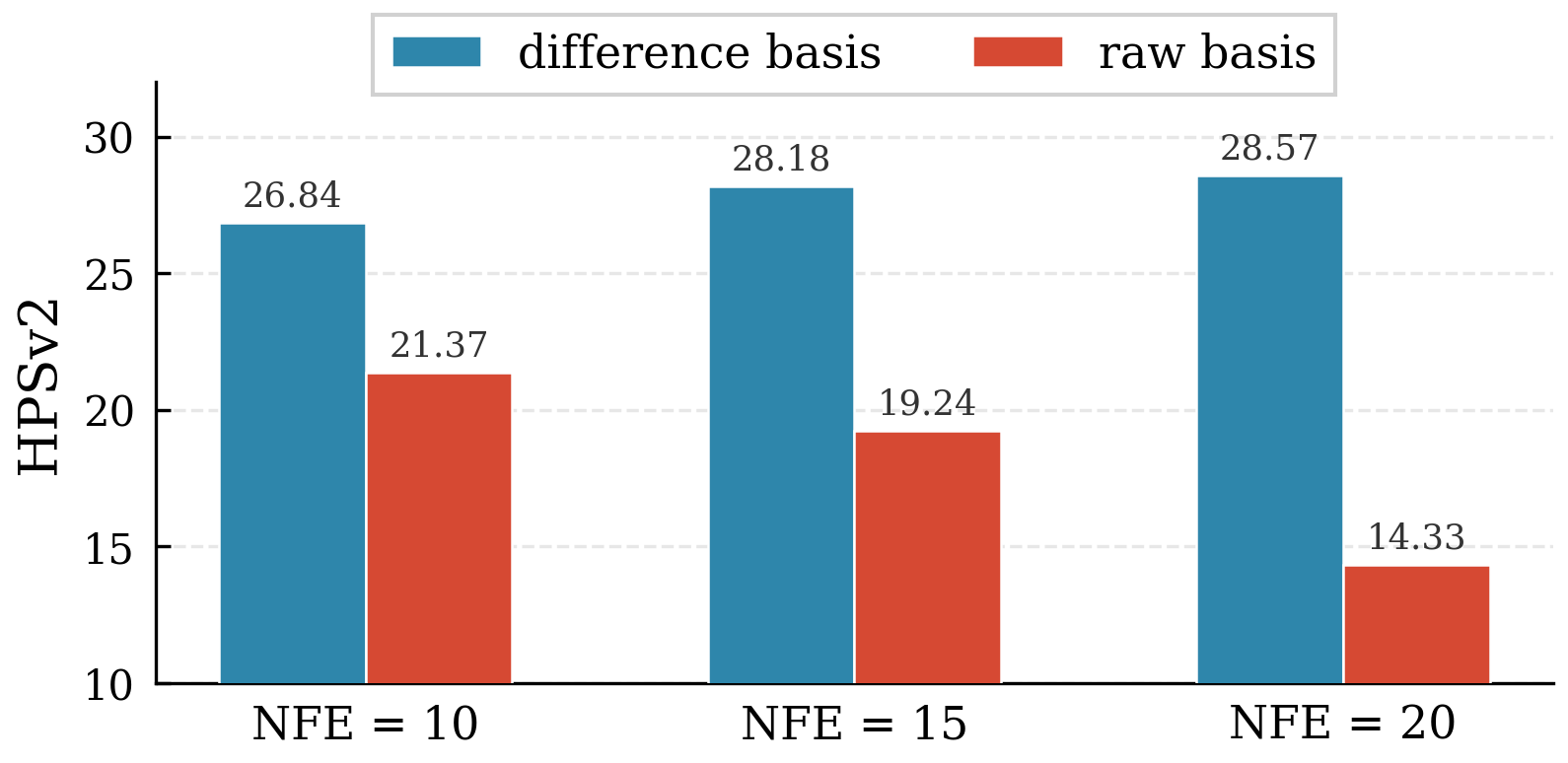}
    \caption{Ablation study on basis construction in Stable Diffusion 3.5-Medium. Our proposed different basis effectively addresses the issue of ill-conditioning.}
    \label{fig:basis_ablation}
    \vspace{-2mm}
\end{figure}

\section{Related Works}
\textit{\textbf{Fast numerical solvers}}. A predominant branch of training-free acceleration methods focuses on improving the numerical solver employed in reverse-time sampling. Early work such as DDIM~\cite{song2020denoising} established the ODE-based paradigm for fast deterministic generation. Subsequent research has developed increasingly effective high-order and multi-step solvers. 
PNDM~\cite{liu2022pseudo} introduced multi-step pseudo numerical methods, showing that leveraging historical samples can significantly improve the quality-speed trade-off over first-order samplers. Its improved implementation, iPNDM~\cite{von-platen-etal-2022-diffusers}, is frequently utilized in modern diffusion sampling. 
DEIS~\cite{zhang2022fast} further exploited the semi-linear structure of diffusion dynamics through an exponential-integrator formulation. Similarly, DPM-Solver~\cite{lu2022dpm} and DPM++~\cite{lu2025dpm} derived dedicated high-order solvers by analytically isolating the linear components, significantly reducing discretization errors, and have emerged as a robust practical standard, particularly for guided sampling scenarios. UniPC~\cite{zhao2023unipc} unified the predictor-corrector paradigm. Subsequent refinements such as DPM-Solver-v3~\cite{zheng2023dpm} and DC-Solver~\cite{zhao2024dc}, further improved solver parameterization and correction stability. LTC-Accel~\cite{zhu2025accelerating} leverages the local transition coherence of generative trajectories to dynamically approximate adjacent network outputs, achieving training-free acceleration. STORK~\cite{tan2025stork} addressed numerical stiffness and removed the reliance on semi-linear assumptions and proposed a stabilized solver. 
Despite their differences, these methods still construct update rules primarily from numerical analysis or step-wise trend heuristics. 
In contrast, our method revisits multi-step solving from a trajectory-geometric perspective and seeks to exploit the intrinsic low-dimensional and cross-trajectory consistent structure of generative trajectories more directly.

\textit{\textbf{Diffusion sampling dynamics}}. Prior research on diffusion sampling dynamics has mainly established trajectory regularity through empirical observations, supported by preliminary theoretical insights under idealized settings. 
Notably, Wang et al.~\cite{wang2023diffusion} observed that individual reverse-time sampling trajectories are effectively low-dimensional, often resembling two-dimensional rotations. 
Chen et al.~\cite{chen2024trajectory} studied deterministic PF-ODE trajectories, finding that each trajectory closely follows the displacement between its endpoints, with residuals neatly captured by two or three orthogonal directions, yielding a characteristic boomerang-like shape. 
Their subsequent work~\cite{chen2025geometric} strengthened the cross-trajectory consistency claim by aligning projected trajectories with orthogonal Procrustes, revealing a nearly canonical trajectory shape after sample-specific rotations and reflections. While these works have empirically uncovered the trajectory regularity of diffusion sampling dynamics, their theoretical grounding relies largely on KDE-based surrogates of the empirical data distribution, analytically tractable Gaussian models, and supportive Gaussian-mixture calculations. 
To the best of our knowledge, prior literature lacks a unified theorem within a Mixture-of-Gaussian framework that jointly formalizes both trajectory low-dimensionality and cross-trajectory isomorphism. In contrast, our work provides a principled derivation that simultaneously establishes these two properties.

\section{Conclusion}
In this paper, we introduced LEAP, a novel, training-free multi-step ODE solver that effectively addresses the performance degradation of diffusion models in low-NFE regimes. Our approach is built upon a rigorous investigation of the geometric properties of diffusion sampling trajectories. By proving the intrinsic low-dimensionality and cross-trajectory consistency of generative paths, we formulated a supervised least-squares optimization within the local trajectory subspace. This allows LEAP to derive optimal, sample-agnostic extrapolation coefficients, achieving accurate future state approximations with minimal computational overhead. Extensive evaluations demonstrate that LEAP serves as a versatile, plug-and-play enhancement for modern large-scale diffusion models. LEAP consistently outperforms state-of-the-art baselines across diverse modalities, including text-to-image generation using Stable Diffusion 3.5 and FLUX.1-Dev and image-to-3D generation using TRELLIS.2. Future work will delve deeper into the geometric properties of diffusion sampling trajectories under practical settings to further align our theoretical framework with realistic, finite-step sampling scenarios.

\begin{acks}
This work is supported by the National Natural Science Foundation of China (No. 62476244), the Starry Night Science Fund of Zhejiang University Shanghai Institute for Advanced Study, China (Grant No: SN-ZJU-SIAS-001) and the advanced computing resources provided by the Supercomputing Center of Hangzhou City University.
\end{acks}
\bibliographystyle{ACM-Reference-Format}
\bibliography{refs}

@String{Computer = "{IEEE} Computer" }

@String{Springer = "Springer-Verlag" }

@article{ho2020denoising,
  title={Denoising diffusion probabilistic models},
  author={Ho, Jonathan and Jain, Ajay and Abbeel, Pieter},
  journal={Advances in neural information processing systems},
  volume={33},
  pages={6840--6851},
  year={2020}
}

@inproceedings{sohl2015deep,
  title={Deep unsupervised learning using nonequilibrium thermodynamics},
  author={Sohl-Dickstein, Jascha and Weiss, Eric and Maheswaranathan, Niru and Ganguli, Surya},
  booktitle={International conference on machine learning},
  pages={2256--2265},
  year={2015},
  organization={pmlr}
}

@article{song2020improved,
  title={Improved techniques for training score-based generative models},
  author={Song, Yang and Ermon, Stefano},
  journal={Advances in neural information processing systems},
  volume={33},
  pages={12438--12448},
  year={2020}
}

@article{song2020score,
  title={Score-based generative modeling through stochastic differential equations},
  author={Song, Yang and Sohl-Dickstein, Jascha and Kingma, Diederik P and Kumar, Abhishek and Ermon, Stefano and Poole, Ben},
  journal={arXiv preprint arXiv:2011.13456},
  year={2020}
}

@inproceedings{rombach2022high,
  title={High-resolution image synthesis with latent diffusion models},
  author={Rombach, Robin and Blattmann, Andreas and Lorenz, Dominik and Esser, Patrick and Ommer, Bj{\"o}rn},
  booktitle={Proceedings of the IEEE/CVF conference on computer vision and pattern recognition},
  pages={10684--10695},
  year={2022}
}

@article{ramesh2022hierarchical,
  title={Hierarchical text-conditional image generation with clip latents},
  author={Ramesh, Aditya and Dhariwal, Prafulla and Nichol, Alex and Chu, Casey and Chen, Mark},
  journal={arXiv preprint arXiv:2204.06125},
  volume={1},
  number={2},
  pages={3},
  year={2022}
}

@article{saharia2022photorealistic,
  title={Photorealistic text-to-image diffusion models with deep language understanding},
  author={Saharia, Chitwan and Chan, William and Saxena, Saurabh and Li, Lala and Whang, Jay and Denton, Emily L and Ghasemipour, Kamyar and Gontijo Lopes, Raphael and Karagol Ayan, Burcu and Salimans, Tim and others},
  journal={Advances in neural information processing systems},
  volume={35},
  pages={36479--36494},
  year={2022}
}

@article{ho2022video,
  title={Video diffusion models},
  author={Ho, Jonathan and Salimans, Tim and Gritsenko, Alexey and Chan, William and Norouzi, Mohammad and Fleet, David J},
  journal={Advances in neural information processing systems},
  volume={35},
  pages={8633--8646},
  year={2022}
}

@article{blattmann2023stable,
  title={Stable video diffusion: Scaling latent video diffusion models to large datasets},
  author={Blattmann, Andreas and Dockhorn, Tim and Kulal, Sumith and Mendelevitch, Daniel and Kilian, Maciej and Lorenz, Dominik and Levi, Yam and English, Zion and Voleti, Vikram and Letts, Adam and others},
  journal={arXiv preprint arXiv:2311.15127},
  year={2023}
}

@article{kong2020diffwave,
  title={Diffwave: A versatile diffusion model for audio synthesis},
  author={Kong, Zhifeng and Ping, Wei and Huang, Jiaji and Zhao, Kexin and Catanzaro, Bryan},
  journal={arXiv preprint arXiv:2009.09761},
  year={2020}
}

@article{chen2020wavegrad,
  title={Wavegrad: Estimating gradients for waveform generation},
  author={Chen, Nanxin and Zhang, Yu and Zen, Heiga and Weiss, Ron J and Norouzi, Mohammad and Chan, William},
  journal={arXiv preprint arXiv:2009.00713},
  year={2020}
}

@article{salimans2022progressive,
  title={Progressive distillation for fast sampling of diffusion models},
  author={Salimans, Tim and Ho, Jonathan},
  journal={arXiv preprint arXiv:2202.00512},
  year={2022}
}

@article{luo2023latent,
  title={Latent consistency models: Synthesizing high-resolution images with few-step inference},
  author={Luo, Simian and Tan, Yiqin and Huang, Longbo and Li, Jian and Zhao, Hang},
  journal={arXiv preprint arXiv:2310.04378},
  year={2023}
}

@article{song2023consistency,
  title={Consistency models},
  author={Song, Yang and Dhariwal, Prafulla and Chen, Mark and Sutskever, Ilya},
  year={2023}
}

@article{kim2023consistency,
  title={Consistency trajectory models: Learning probability flow ode trajectory of diffusion},
  author={Kim, Dongjun and Lai, Chieh-Hsin and Liao, Wei-Hsiang and Murata, Naoki and Takida, Yuhta and Uesaka, Toshimitsu and He, Yutong and Mitsufuji, Yuki and Ermon, Stefano},
  journal={arXiv preprint arXiv:2310.02279},
  year={2023}
}

@inproceedings{esser2024scaling,
  title={Scaling rectified flow transformers for high-resolution image synthesis},
  author={Esser, Patrick and Kulal, Sumith and Blattmann, Andreas and Entezari, Rahim and M{\"u}ller, Jonas and Saini, Harry and Levi, Yam and Lorenz, Dominik and Sauer, Axel and Boesel, Frederic and others},
  booktitle={Forty-first international conference on machine learning},
  year={2024}
}

@misc{flux1dev,
  author       = {{Black Forest Labs}},
  title        = {{FLUX.1} [dev]},
  year         = {2024},
  publisher    = {GitHub},
  journal      = {GitHub repository},
  howpublished = {\url{https://github.com/black-forest-labs/flux}}
}

@article{lu2022dpm,
  title={Dpm-solver: A fast ode solver for diffusion probabilistic model sampling in around 10 steps},
  author={Lu, Cheng and Zhou, Yuhao and Bao, Fan and Chen, Jianfei and Li, Chongxuan and Zhu, Jun},
  journal={Advances in neural information processing systems},
  volume={35},
  pages={5775--5787},
  year={2022}
}

@article{song2020denoising,
  title={Denoising diffusion implicit models},
  author={Song, Jiaming and Meng, Chenlin and Ermon, Stefano},
  journal={arXiv preprint arXiv:2010.02502},
  year={2020}
}

@article{zhao2023unipc,
  title={Unipc: A unified predictor-corrector framework for fast sampling of diffusion models},
  author={Zhao, Wenliang and Bai, Lujia and Rao, Yongming and Zhou, Jie and Lu, Jiwen},
  journal={Advances in Neural Information Processing Systems},
  volume={36},
  pages={49842--49869},
  year={2023}
}

@article{liu2022pseudo,
  title={Pseudo numerical methods for diffusion models on manifolds},
  author={Liu, Luping and Ren, Yi and Lin, Zhijie and Zhao, Zhou},
  journal={arXiv preprint arXiv:2202.09778},
  year={2022}
}

@inproceedings{zhou2024fast,
  title={Fast ode-based sampling for diffusion models in around 5 steps},
  author={Zhou, Zhenyu and Chen, Defang and Wang, Can and Chen, Chun},
  booktitle={Proceedings of the IEEE/CVF Conference on Computer Vision and Pattern Recognition},
  pages={7777--7786},
  year={2024}
}

@article{chen2025geometric,
  title={Geometric regularity in deterministic sampling dynamics of diffusion-based generative models},
  author={Chen, Defang and Zhou, Zhenyu and Wang, Can and Lyu, Siwei},
  journal={Journal of Statistical Mechanics: Theory and Experiment},
  volume={2025},
  number={12},
  pages={124002},
  year={2025},
  publisher={IOP Publishing}
}

@inproceedings{xue2024accelerating,
  title={Accelerating diffusion sampling with optimized time steps},
  author={Xue, Shuchen and Liu, Zhaoqiang and Chen, Fei and Zhang, Shifeng and Hu, Tianyang and Xie, Enze and Li, Zhenguo},
  booktitle={Proceedings of the IEEE/CVF Conference on Computer Vision and Pattern Recognition},
  pages={8292--8301},
  year={2024}
}

@article{sabour2024align,
  title={Align your steps: Optimizing sampling schedules in diffusion models},
  author={Sabour, Amirmojtaba and Fidler, Sanja and Kreis, Karsten},
  journal={arXiv preprint arXiv:2404.14507},
  year={2024}
}

@inproceedings{ma2024deepcache,
  title={Deepcache: Accelerating diffusion models for free},
  author={Ma, Xinyin and Fang, Gongfan and Wang, Xinchao},
  booktitle={Proceedings of the IEEE/CVF conference on computer vision and pattern recognition},
  pages={15762--15772},
  year={2024}
}

@article{li2024faster,
  title={Faster diffusion: Rethinking the role of the encoder for diffusion model inference},
  author={Li, Senmao and Hu, Taihang and van de Weijer, Joost and Khan, Fahad S and Liu, Tao and Li, Linxuan and Yang, Shiqi and Wang, Yaxing and Cheng, Ming-Ming and Yang, Jian},
  journal={Advances in Neural Information Processing Systems},
  volume={37},
  pages={85203--85240},
  year={2024}
}

@inproceedings{yu2025ab,
  title={Ab-cache: Training-free acceleration of diffusion models via adams-bashforth cached feature reuse},
  author={Yu, Zichao and Zou, Zhen and Shao, Guojiang and Zhang, Chenwei and Xu, Shengze and Huang, Jie and Zhao, Feng and Cun, Xiaodong and Zhang, Wenyi},
  booktitle={Proceedings of the 33rd ACM International Conference on Multimedia},
  pages={10408--10417},
  year={2025}
}

@inproceedings{wimbauer2024cache,
  title={Cache me if you can: Accelerating diffusion models through block caching},
  author={Wimbauer, Felix and Wu, Bichen and Schoenfeld, Edgar and Dai, Xiaoliang and Hou, Ji and He, Zijian and Sanakoyeu, Artsiom and Zhang, Peizhao and Tsai, Sam and Kohler, Jonas and others},
  booktitle={Proceedings of the IEEE/CVF Conference on Computer Vision and Pattern Recognition},
  pages={6211--6220},
  year={2024}
}

@inproceedings{wan2025elucidating,
  title={Elucidating flow matching ODE dynamics via data geometry and denoisers},
  author={Wan, Zhengchao and Wang, Qingsong and Mishne, Gal and Wang, Yusu},
  booktitle={Forty-second International Conference on Machine Learning},
  year={2025}
}

@article{wang2024diffusion,
  title={Diffusion Sampling Correction via Approximately 10 Parameters},
  author={Wang, Guangyi and Peng, Wei and Li, Lijiang and Chen, Wenyu and Cai, Yuren and Su, Songzhi},
  journal={arXiv preprint arXiv:2411.06503},
  year={2024}
}

@book{oksendal2013stochastic,
  title={Stochastic differential equations: an introduction with applications},
  author={Oksendal, Bernt},
  year={2013},
  publisher={Springer Science \& Business Media}
}

@article{lipman2022flow,
  title={Flow matching for generative modeling},
  author={Lipman, Yaron and Chen, Ricky TQ and Ben-Hamu, Heli and Nickel, Maximilian and Le, Matt},
  journal={arXiv preprint arXiv:2210.02747},
  year={2022}
}

@article{liu2022flow,
  title={Flow straight and fast: Learning to generate and transfer data with rectified flow},
  author={Liu, Xingchao and Gong, Chengyue and Liu, Qiang},
  journal={arXiv preprint arXiv:2209.03003},
  year={2022}
}

@article{albergo2022building,
  title={Building normalizing flows with stochastic interpolants},
  author={Albergo, Michael S and Vanden-Eijnden, Eric},
  journal={arXiv preprint arXiv:2209.15571},
  year={2022}
}

@article{lu2025dpm,
  title={Dpm-solver++: Fast solver for guided sampling of diffusion probabilistic models},
  author={Lu, Cheng and Zhou, Yuhao and Bao, Fan and Chen, Jianfei and Li, Chongxuan and Zhu, Jun},
  journal={Machine Intelligence Research},
  volume={22},
  number={4},
  pages={730--751},
  year={2025},
  publisher={Springer}
}

@article{zheng2023dpm,
  title={Dpm-solver-v3: Improved diffusion ode solver with empirical model statistics},
  author={Zheng, Kaiwen and Lu, Cheng and Chen, Jianfei and Zhu, Jun},
  journal={Advances in Neural Information Processing Systems},
  volume={36},
  pages={55502--55542},
  year={2023}
}

@inproceedings{zhao2024dc,
  title={Dc-solver: Improving predictor-corrector diffusion sampler via dynamic compensation},
  author={Zhao, Wenliang and Wang, Haolin and Zhou, Jie and Lu, Jiwen},
  booktitle={European Conference on Computer Vision},
  pages={450--466},
  year={2024},
  organization={Springer}
}

@article{zhang2022fast,
  title={Fast sampling of diffusion models with exponential integrator},
  author={Zhang, Qinsheng and Chen, Yongxin},
  journal={arXiv preprint arXiv:2204.13902},
  year={2022}
}

@article{xu2023imagereward,
  title={Imagereward: Learning and evaluating human preferences for text-to-image generation},
  author={Xu, Jiazheng and Liu, Xiao and Wu, Yuchen and Tong, Yuxuan and Li, Qinkai and Ding, Ming and Tang, Jie and Dong, Yuxiao},
  journal={Advances in Neural Information Processing Systems},
  volume={36},
  pages={15903--15935},
  year={2023}
}

@inproceedings{radford2021learning,
  title={Learning transferable visual models from natural language supervision},
  author={Radford, Alec and Kim, Jong Wook and Hallacy, Chris and Ramesh, Aditya and Goh, Gabriel and Agarwal, Sandhini and Sastry, Girish and Askell, Amanda and Mishkin, Pamela and Clark, Jack and others},
  booktitle={International conference on machine learning},
  pages={8748--8763},
  year={2021},
  organization={PmLR}
}

@article{wu2023human,
  title={Human preference score v2: A solid benchmark for evaluating human preferences of text-to-image synthesis},
  author={Wu, Xiaoshi and Hao, Yiming and Sun, Keqiang and Chen, Yixiong and Zhu, Feng and Zhao, Rui and Li, Hongsheng},
  journal={arXiv preprint arXiv:2306.09341},
  year={2023}
}

@misc{von-platen-etal-2022-diffusers,
  author = {Patrick von Platen and Suraj Patil and Anton Lozhkov and Pedro Cuenca and Nathan Lambert and Kashif Rasul and Mishig Davaadorj and Thomas Wolf},
  title = {Diffusers: State-of-the-art diffusion models},
  year = {2022},
  publisher = {GitHub},
  journal = {GitHub repository},
  howpublished = {\url{https://github.com/huggingface/diffusers}}
}

@inproceedings{chen2024trajectory,
  title={On the Trajectory Regularity of ODE-based Diffusion Sampling},
  author={Chen, Defang and Zhou, Zhenyu and Wang, Can and Shen, Chunhua and Lyu, Siwei},
  booktitle={International Conference on Machine Learning},
  pages={7905--7934},
  year={2024}
}

@inproceedings{zhu2025accelerating,
  title={Accelerating Diffusion Sampling via Exploiting Local Transition Coherence},
  author={Zhu, Shangwen and Zhang, Han and Yang, Zhantao and Peng, Qianyu and Pu, Zhao and Wang, Huangji and Cheng, Fan},
  booktitle={Proceedings of the IEEE/CVF International Conference on Computer Vision},
  pages={18284--18293},
  year={2025}
}

@article{tan2025stork,
  title={STORK: Faster Diffusion And Flow Matching Sampling By Resolving Both Stiffness And Structure-Dependence},
  author={Tan, Zheng and Wang, Weizhen and Bertozzi, Andrea L and Ryu, Ernest K},
  journal={arXiv preprint arXiv:2505.24210},
  year={2025}
}

@article{xiang2025native,
  title={Native and compact structured latents for 3d generation},
  author={Xiang, Jianfeng and Chen, Xiaoxue and Xu, Sicheng and Wang, Ruicheng and Lv, Zelong and Deng, Yu and Zhu, Hongyuan and Dong, Yue and Zhao, Hao and Yuan, Nicholas Jing and others},
  journal={arXiv preprint arXiv:2512.14692},
  year={2025}
}

@inproceedings{downs2022google,
  title={Google scanned objects: A high-quality dataset of 3d scanned household items},
  author={Downs, Laura and Francis, Anthony and Koenig, Nate and Kinman, Brandon and Hickman, Ryan and Reymann, Krista and McHugh, Thomas B and Vanhoucke, Vincent},
  booktitle={2022 International Conference on Robotics and Automation (ICRA)},
  pages={2553--2560},
  year={2022},
  organization={Ieee}
}

@article{wang2023diffusion,
  title={Diffusion models generate images like painters: an analytical theory of outline first, details later},
  author={Wang, Binxu and Vastola, John J},
  journal={arXiv preprint arXiv:2303.02490},
  year={2023}
}

@article{hunyuan3d2025hunyuan3d,
  title={Hunyuan3d 2.1: From images to high-fidelity 3d assets with production-ready pbr material},
  author={Hunyuan3D, Team and Yang, Shuhui and Yang, Mingxin and Feng, Yifei and Huang, Xin and Zhang, Sheng and He, Zebin and Luo, Di and Liu, Haolin and Zhao, Yunfei and others},
  journal={arXiv preprint arXiv:2506.15442},
  year={2025}
}

\appendix

\clearpage
\section{Proofs for Section~\ref{sec:geometry}}
\label{app:trajectory_consistency_proofs}
This appendix section proves Theorems and Corollaries from Section~\ref{sec:geometry}. First, we provide an exact decomposition of the PF-ODE under a mild assumption. Next, we prove the deterministic pairwise alignment bound. Finally, we apply the Gaussian concentration theorem to get a high-probability bound.

\subsection{Notations}
Consider the variance-preserving forward diffusion situation:
\begin{equation}
d\mathbf{X}_t = -\frac12 \beta(t) \mathbf{X}_t\,dt + \sqrt{\beta(t)}\,d\mathbf{W}_t,
\qquad t\in[0,T],
\end{equation}
where \(\beta:[0,T]\to(0,\infty)\) is integrable. We define
\begin{equation}
\alpha_t := \exp\!\Big(-\frac12\int_0^t \beta(s)\,ds\Big),
\qquad
\nu_t^2 := \alpha_t^2 \sigma_0^2 + (1-\alpha_t^2).
\end{equation}
Suppose data distribution satisfies GMM assumption:
\begin{equation}
p_0(\mathbf{x}) = \sum_{m=1}^M \pi_m \mathcal N(\boldsymbol{\mu}_m,\sigma_0^2 \mathbf{I}_d),
\qquad
\pi_m>0,
\qquad
\sum_{m=1}^M \pi_m=1.
\end{equation}
To describe the layout of the mixed distribution center, we define the space formed by their means as
\begin{equation}
U:=\operatorname{span}\{\boldsymbol{\mu}_1,\dots,\boldsymbol{\mu}_M\},
\qquad
r:=\dim(U),
\qquad
R:=\max_{1\le m\le M}\|\boldsymbol{\mu}_m\|.
\end{equation}
We also let \(P\) to be the orthogonal projection onto \(U\), with \(Q:=\mathbf{I}-P\).

\subsection{Structural lemmas}
We now begin to derive the closed-form expressions of PF-ODE dynamics. We start our proofs by formulating the marginal probability density of the forward diffusion.

\begin{lemma}[Closed Form of Marginal Probability Function]
\label{lem:noised_marginal}
For every \(t\in[0,T]\), we have
\begin{equation}
p_t(\mathbf{x}) = \sum_{s=1}^S \pi_s \mathcal N(\alpha_t\boldsymbol{\mu}_s,\nu_t^2 \mathbf{I}_d).
\end{equation}
\end{lemma}

\begin{proof}
Suppose \(C\in\{1,\dots,S\}\) satisfy condition \(\mathbb P(C=s)=\pi_s\), and suppose the conditional probability density of the initial state is
\begin{equation}
    p(\mathbf{x}_0 \mid C=s) = \mathcal{N}(\mathbf{x}_0; \boldsymbol{\mu}_s, \sigma_0^2 \mathbf{I}_d).
\end{equation}
Because forward diffusion process defines the transition probability density:
\begin{equation}
    p(\mathbf{x}_t \mid \mathbf{x}_0) = \mathcal{N}(\mathbf{x}_t; \alpha_t \mathbf{x}_0, (1-\alpha_t^2)\mathbf{I}_d).
\end{equation}
Hence, conditional on \(C=s\), the probability density of \(\mathbf{x}_t\) is obtained by the following equation:
\begin{equation}
    \begin{aligned}
    p(\mathbf{x}_t \mid C=s) &= \int p(\mathbf{x}_t \mid \mathbf{x}_0) p(\mathbf{x}_0 \mid C=s) \, \mathrm{d}\mathbf{x}_0\\
    &= 
    \mathcal{N}\!\bigl(\mathbf{x}_t; \alpha_t\boldsymbol{\mu}_s, (\alpha_t^2\sigma_0^2 + 1-\alpha_t^2)\mathbf{I}_d\bigr) \\
    &= 
    \mathcal{N}(\mathbf{x}_t; \alpha_t\boldsymbol{\mu}_s, \nu_t^2 \mathbf{I}_d).
    \end{aligned}
\end{equation}
Averaging over \(s\) gives the claim.
\end{proof}

\begin{lemma}[Closed-form of score function]
\label{lem:score_formula}
Define
\begin{equation}
w_s(\mathbf{x},t):=
\frac{\pi_s \exp\!\Big(\frac{\alpha_t}{\nu_t^2}\langle \mathbf{x},\boldsymbol{\mu}_s\rangle
-\frac{\alpha_t^2}{2\nu_t^2}\|\boldsymbol{\mu}_s\|^2\Big)}
{\sum_{j=1}^S \pi_j \exp\!\Big(\frac{\alpha_t}{\nu_t^2}\langle \mathbf{x},\boldsymbol{\mu}_j\rangle
-\frac{\alpha_t^2}{2\nu_t^2}\|\boldsymbol{\mu}_j\|^2\Big)},
\end{equation}
and
\begin{equation}
\mathbf{m}_t(\mathbf{x}):=\sum_{s=1}^S w_s(\mathbf{x},t)\boldsymbol{\mu}_s.
\end{equation}
Then
\begin{equation}
\nabla \log p_t(\mathbf{x})
=
-\frac{1}{\nu_t^2}\bigl(\mathbf{x}-\alpha_t \mathbf{m}_t(\mathbf{x})\bigr).
\end{equation}
\end{lemma}

\begin{proof}
By Lemma~\ref{lem:noised_marginal},
\begin{equation}
p_t(\mathbf{x})=\sum_{s=1}^S \pi_s \phi_{\nu_t^2 \mathbf{I}_d}(\mathbf{x}-\alpha_t\boldsymbol{\mu}_s),
\end{equation}
where \(\phi_{\mathbf{\Sigma}}\) denotes the Gaussian density with covariance \(\mathbf{\Sigma}\). 
Therefore,
\begin{equation}
\nabla p_t(\mathbf{x})
=
-\frac{1}{\nu_t^2}
\sum_{s=1}^S \pi_s (\mathbf{x}-\alpha_t\boldsymbol{\mu}_s)\phi_{\nu_t^2 \mathbf{I}_d}(\mathbf{x}-\alpha_t\boldsymbol{\mu}_s).
\label{eq:logpt}
\end{equation}
Therefore, the closed-form of score function can be derived by dividing Eq.~\ref{eq:logpt} by \(p_t(\mathbf{x})\):
\begin{equation}
\nabla \log p_t(\mathbf{x})
=
-\frac{1}{\nu_t^2}
\Bigl(\mathbf{x}-\alpha_t\sum_{s=1}^S w_s(\mathbf{x},t)\boldsymbol{\mu}_s\Bigr)
=
-\frac{1}{\nu_t^2}\bigl(\mathbf{x}-\alpha_t \mathbf{m}_t(\mathbf{x})\bigr).
\end{equation}
\end{proof}

\begin{lemma}[Sufficiency of Mean Subspace]
\label{lem:subspace_dependence}
For every \(\mathbf{x}\in\mathbb R^d\) and \(t\in[0,T]\),
\begin{equation}
\label{eq:P_invariance}
w_s(\mathbf{x},t)=w_s(P\mathbf{x},t),
\qquad
\mathbf{m}_t(\mathbf{x})=\mathbf{m}_t(P\mathbf{x})\in U.
\end{equation}

Furthermore, let \(\mathbf{x}_{t\in[0,T]}\) be a PF-ODE trajectory. Its derivative with respect to \(t\) is given by
\begin{equation}
\label{eq:PF-ODE}
\frac{d\mathbf{x}_t}{dt}=a_t \mathbf{x}_t+b_t \mathbf{m}_t(P\mathbf{x}_t),
\end{equation}
where
\begin{equation}
a_t:=\frac12\beta(t)\Bigl(1-\nu_t^{-2}\Bigr),
\qquad
b_t:=\frac12\beta(t)\alpha_t\nu_t^{-2}.
\end{equation}
\end{lemma}

\begin{proof}
Note that the expression for \(w_s(\mathbf{x},t)\) only depends on \(\mathbf{x}\) through the inner products \(\langle \mathbf{x},\boldsymbol{\mu}_s\rangle\). Since for \(\boldsymbol{\mu}_s\in U\), we have 
\begin{equation}
\langle \mathbf{x},\boldsymbol{\mu}_s\rangle = \langle P\mathbf{x},\boldsymbol{\mu}_s\rangle,
\end{equation}
which leads to \(w_s(\mathbf{x},t)=w_s(P\mathbf{x},t)\). Furthermore, it is obvious that $\mathbf{m}_t(x) \in U$ because it is simply a sum of vectors $\mu_s \in U$:
\begin{equation}
\mathbf{m}_t(\mathbf{x})=\sum_{s=1}^S w_s(\mathbf{x},t)\boldsymbol{\mu}_s.
\end{equation}
In forward process, the PF-ODE is
\begin{equation}
\begin{aligned}
\frac{d\mathbf{x}_t}{dt}
&=
-\frac12\beta(t)\mathbf{x}_t
-\frac12\beta(t)\nabla\log p_t(\mathbf{x}_t)\\
&=
\frac12\beta(t)
\Bigl[
\mathbf{x}_t-\nu_t^{-2}\mathbf{x}_t+\alpha_t\nu_t^{-2}\mathbf{m}_t(\mathbf{x}_t)
\Bigr].
\end{aligned}
\end{equation}
Using Eq.~\ref{eq:P_invariance} immediately yields Eq.\ref{eq:PF-ODE}.
\end{proof}

\begin{lemma}[Exact \(U\oplus U^\perp\) decomposition]
\label{lem:exact_decomposition}
Let
\begin{equation}
\Phi(t_2,t_1):=\exp\!\Big(\int_{t_1}^{t_2} a_v\,dv\Big),
\qquad
\rho_t:=\Phi(t,T).
\end{equation}
Then we obtain the following equation by integrating from \(T\) to \(t\):
\begin{equation}
Q\mathbf{x}_t=\rho_t Q\mathbf{x}_T,
\end{equation}
and
\begin{equation}
P\mathbf{x}_t
=
\rho_t P\mathbf{x}_T
-
\int_t^T \Phi(t,\tau)b_\tau\,\mathbf{m}_\tau(P\mathbf{x}_\tau)\,d\tau.
\end{equation}
As a result,
\begin{equation}
\label{eq:Dt_formula}
Q\mathbf{D}_t=(\rho_t-1)Q\mathbf{x}_T.
\end{equation}
\end{lemma}

\begin{proof}
Apply \(Q\) to Eq.~\ref{eq:PF-ODE}. Since \(\mathbf{m}_t(P\mathbf{x}_t)\in U\), we obtain
\begin{equation}
Q\frac{d\mathbf{x}_t}{dt}=a_t Q\mathbf{x}_t,
\end{equation}
whose exact solution is
\begin{equation}
Q\mathbf{x}_t=\exp\!\Big(\int_T^t a_v\,dv\Big) Q\mathbf{x}_T = \rho_t Q\mathbf{x}_T.
\end{equation}
Similarly, for the projection $P$, we can obtain that
\begin{equation}
P\frac{d\mathbf{x}_t}{dt}=a_t P\mathbf{x}_t+b_t \mathbf{m}_t(P\mathbf{x}_t),
\end{equation}
and the variation-of-constants formula (integrating from \(T\) to \(t\)) leads to
\begin{equation}
P\mathbf{x}_t
=
\rho_t P\mathbf{x}_T
+
\int_T^t \Phi(t,\tau)b_\tau\,\mathbf{m}_\tau(P\mathbf{x}_\tau)\,d\tau.
\end{equation}
Subtracting \(Q\mathbf{x}_T\) immediately gives the Eq.~\ref{eq:Dt_formula}.
\end{proof}

\begin{lemma}[Mean-Subspace Displacement Bound]
\label{lem:mean_subspace_bound}
Define
\begin{equation}
A_T:=\sup_{0\le t\le T}|\rho_t-1|,
\qquad
B_T:=\sup_{0\le t\le T}\int_t^T \Phi(t,\tau)|b_\tau|\,d\tau.
\end{equation}
Then every trajectory satisfies
\begin{equation}
\sup_{0\le t\le T}\|P\mathbf{D}_t\|
\le
A_T\|P\mathbf{x}_T\| + B_T R.
\end{equation}
\end{lemma}

\begin{proof}
By Lemma~\ref{lem:exact_decomposition},
\begin{equation}
P\mathbf{D}_t
=
(\rho_t-1)P\mathbf{x}_T
-
\int_t^T \Phi(t,\tau)b_\tau\,\mathbf{m}_\tau(P\mathbf{x}_\tau)\,d\tau.
\end{equation}
Since \(\mathbf{m}_\tau(P\mathbf{x}_\tau)\) is a convex combination of \(\boldsymbol{\mu}_1,\dots,\boldsymbol{\mu}_M\), it's obvious that
\begin{equation}
\|\mathbf{m}_\tau(P\mathbf{x}_\tau)\|\le R.
\end{equation}
Hence by using the triangle inequality, we have
\begin{equation}
\|P\mathbf{D}_t\|
\le
|\rho_t-1|\|P\mathbf{x}_T\|
+
\int_t^T \Phi(t,\tau)|b_\tau|\,R\,d\tau.
\end{equation}
Taking the supremum over \(t\in[0,T]\) proves the lemma.
\end{proof}

\begin{lemma}[Minimum orthogonal alignment error]
\label{lem:best_orthogonal_alignment}
For any \(\mathbf{u},\mathbf{v}\in\mathbb R^m\),
\begin{equation}
\min_{\mathbf{H}\in O(m)} \|\mathbf{u}-\mathbf{H}\mathbf{v}\| = \bigl|\|\mathbf{u}\|-\|\mathbf{v}\|\bigr|.
\end{equation}
\end{lemma}

\begin{proof}
For any \(\mathbf{H}\in O(m)\), the inequality
\begin{equation}
\|\mathbf{u}-\mathbf{H}\mathbf{v}\|\ge \bigl|\|\mathbf{u}\|-\|\mathbf{H}\mathbf{v}\|\bigr|=\bigl|\|\mathbf{u}\|-\|\mathbf{v}\|\bigr|
\end{equation}
always holds. Conversely, if \(\mathbf{u},\mathbf{v}\neq 0\), choose \(\mathbf{H}\in O(m)\) such that
\begin{equation}
\mathbf{H}\frac{\mathbf{v}}{\|\mathbf{v}\|}=\frac{\mathbf{u}}{\|\mathbf{u}\|}.
\end{equation}
Then \(\mathbf{H}\mathbf{v}=(\|\mathbf{v}\|/\|\mathbf{u}\|)\mathbf{u}\), so
\begin{equation}
\|\mathbf{u}-\mathbf{H}\mathbf{v}\|=\Bigl\|\mathbf{u}-\frac{\|\mathbf{v}\|}{\|\mathbf{u}\|}\mathbf{u}\Bigr\|=\bigl|\|\mathbf{u}\|-\|\mathbf{v}\|\bigr|.
\end{equation}
Finally, if \(\mathbf{u}=0\) or \(\mathbf{v}=0\), the result holds trivially.
\end{proof}

\subsection{A uniform high-probability bound}
Building upon the deterministic trajectory decomposition provided in the previous subsection, this part gives a uniform high-probability bound for pairwise trajectory alignment proof.
\begin{proposition}[Deterministic pairwise alignment]
\label{prop:deterministic_pairwise_alignment}
Suppose \(\mathbf{x}_{t \in [0,T]}^{(1)}\) and \(\mathbf{x}_{t \in [0,T]}^{(2)}\) are two exact PF-ODE trajectories with starting noise \(\mathbf{x}_T^{(1)}\) and \(\mathbf{x}_T^{(2)}\). Then there exists \(\mathbf{H}\in O(d)\) such that
\begin{equation}
\begin{aligned}
\sup_{0\le t\le T}
\|\mathbf{D}_t^{(1)}-\mathbf{H} \mathbf{D}_t^{(2)}\|
\le
A_T\bigl|&\|Q\mathbf{x}_T^{(1)}\|-\|Q\mathbf{x}_T^{(2)}\|\bigr| \\
&+
A_T\bigl(\|P\mathbf{x}_T^{(1)}\|+\|P\mathbf{x}_T^{(2)}\|\bigr)
+
2B_T R.
\end{aligned}
\end{equation}
Consequently,
\begin{equation}
\begin{aligned}
\max_{0\le n\le N}
\|\mathbf{D}_n^{(1)}-\mathbf{H} \mathbf{D}_n^{(2)}\|
\le
A_T\bigl|&\|Q\mathbf{x}_T^{(1)}\|-\|Q\mathbf{x}_T^{(2)}\|\bigr| \\
&+
A_T\bigl(\|P\mathbf{x}_T^{(1)}\|+\|P\mathbf{x}_T^{(2)}\|\bigr)
+
2B_T R.
\end{aligned}
\end{equation}
\end{proposition}

\begin{proof}
Decompose \(\mathbb R^d=U\oplus U^\perp\). By Lemma~\ref{lem:best_orthogonal_alignment}, there exists \(\mathbf{H}_Q\in O(U^\perp)\) such that
\begin{equation}
\|Q\mathbf{x}_T^{(1)}-\mathbf{H}_Q Q\mathbf{x}_T^{(2)}\|
=
\bigl|\|Q\mathbf{x}_T^{(1)}\|-\|Q\mathbf{x}_T^{(2)}\|\bigr|.
\end{equation}
Choose any \(\mathbf{H}_P\in O(U)\) and define
\begin{equation}
\mathbf{H}:=\mathbf{H}_P\oplus \mathbf{H}_Q\in O(d).
\end{equation}

By Lemma~\ref{lem:exact_decomposition},
\begin{equation}
Q\mathbf{D}_t^{(k)}=(\rho_t-1)Q\mathbf{x}_T^{(k)}.
\end{equation}
Hence
\begin{equation}
Q\bigl(\mathbf{D}_t^{(1)}-\mathbf{H}\mathbf{D}_t^{(2)}\bigr)
=
(\rho_t-1)\bigl(Q\mathbf{x}_T^{(1)}-\mathbf{H}_Q Q\mathbf{x}_T^{(2)}\bigr),
\end{equation}
so
\begin{equation}
\sup_{0\le t\le T}
\|Q(\mathbf{D}_t^{(1)}-\mathbf{H}\mathbf{D}_t^{(2)})\|
\le
A_T\bigl|\|Q\mathbf{x}_T^{(1)}\|-\|Q\mathbf{x}_T^{(2)}\|\bigr|.
\end{equation}

For the \(U\)-component, Lemma~\ref{lem:mean_subspace_bound} gives
\begin{equation}
\begin{aligned}
\|P(\mathbf{D}_t^{(1)}-\mathbf{H}\mathbf{D}_t^{(2)})\|
&\le
\|P\mathbf{D}_t^{(1)}\|+\|P\mathbf{D}_t^{(2)}\| \\
&\le
A_T\|P\mathbf{x}_T^{(1)}\|+B_TR
+
A_T\|P\mathbf{x}_T^{(2)}\|+B_TR.
\end{aligned}
\end{equation}
Combining the \(U\) and \(U^\perp\) parts via the triangle inequality proves the continuous-time bound. The discrete-time bound follows from
\begin{equation}
\max_{0\le n\le N}\|\mathbf{D}_n^{(1)}-\mathbf{H}\mathbf{D}_n^{(2)}\|
\le
\sup_{0\le t\le T}\|\mathbf{D}_t^{(1)}-\mathbf{H}\mathbf{D}_t^{(2)}\|.
\end{equation}
\end{proof}

\begin{lemma}[Gaussian norm concentration]
\label{lem:gaussian_norm_concentration}
If \(\mathbf{g}\sim\mathcal N(0,\mathbf{I}_m)\), then for every \(t>0\),
\begin{equation}
\mathbb P\bigl(|\|\mathbf{g}\|-\mathbb E\|\mathbf{g}\||\ge t\bigr)\le 2e^{-t^2/2},
\qquad
\mathbb E\|\mathbf{g}\|\le \sqrt m.
\end{equation}
\end{lemma}

\begin{proof}
The map \(\mathbf{g}\mapsto \|\mathbf{g}\|\) is \(1\)-Lipschitz, so the first inequality is the standard Gaussian concentration inequality. The second follows from Jensen's inequality:
\begin{equation}
\mathbb E\|\mathbf{g}\|
\le
\sqrt{\mathbb E\|\mathbf{g}\|^2}
=
\sqrt m.
\end{equation}
\end{proof}

\begin{theorem}[Explicit version of Theorem~\ref{thm:main_consistency}]
\label{thm:explicit_consistency}
Suppose that \(\mathbf{x}_T^{(1)},\dots,\mathbf{x}_T^{(M)}\) are i.i.d. samples from the exact noised marginal \(p_T\). Then for every \(\delta\in(0,1)\), with probability at least \(1-\delta\),
\begin{equation}
\begin{aligned}
\max_{1\le a,b\le M}\;
\inf_{\mathbf{H}\in O(d)}
\max_{0\le n\le N}
&\|\mathbf{D}_n^{(a)}-\mathbf{H} \mathbf{D}_n^{(b)}\| \\
&\le
2\sqrt{2}\,A_T\nu_T\sqrt{\log\frac{4M}{\delta}} \\
&+
2A_T\Bigl(\alpha_T R+\nu_T\Bigl[\sqrt r+\sqrt{2\log\frac{2M}{\delta}}\Bigr]\Bigr)
+
2B_T R.
\end{aligned}
\end{equation}
\end{theorem}

\begin{proof}
By Lemma~\ref{lem:noised_marginal}, the noised marginal at time \(T\) is
\begin{equation}
p_T(\mathbf{x})=\sum_{m=1}^M \pi_m \mathcal N(\alpha_T\boldsymbol{\mu}_m,\nu_T^2 \mathbf{I}_d).
\end{equation}
Hence each sample can be represented as
\begin{equation}
\mathbf{x}_T^{(k)}=\alpha_T\boldsymbol{\mu}_{C_k}+\nu_T \mathbf{g}^{(k)},
\end{equation}
where \(C_k\in\{1,\dots,M\}\), \(\mathbb P(C_k=m)=\pi_m\), and \(\mathbf{g}^{(k)}\sim \mathcal N(0,\mathbf{I}_d)\). Since \(\boldsymbol{\mu}_{C_k}\in U\),
\begin{equation}
Q\mathbf{x}_T^{(k)}=\nu_T Q\mathbf{g}^{(k)}\sim \mathcal N(0,\nu_T^2 \mathbf{I}_{d-r}),
\end{equation}
and
\begin{equation}
P\mathbf{x}_T^{(k)}=\alpha_T\boldsymbol{\mu}_{C_k}+\nu_T P\mathbf{g}^{(k)}.
\end{equation}

We first control the \(Q\)-component. By Lemma~\ref{lem:gaussian_norm_concentration} and a union bound, with probability at least \(1-\delta/2\),
\begin{equation}
\max_{1\le k\le M}
\Bigl|\|Q\mathbf{x}_T^{(k)}\|-\mathbb E\|Q\mathbf{x}_T^{(k)}\|\Bigr|
\le
\nu_T\sqrt{2\log\frac{4M}{\delta}}.
\end{equation}
Therefore, on this event,
\begin{equation}
\max_{1\le a,b\le M}
\bigl|\|Q\mathbf{x}_T^{(a)}\|-\|Q\mathbf{x}_T^{(b)}\|\bigr|
\le
2\nu_T\sqrt{2\log\frac{4M}{\delta}}.
\end{equation}

Next we control the \(P\)-component. Since \(P\mathbf{g}^{(k)}\sim \mathcal N(0,\mathbf{I}_r)\) on \(U\), Lemma~\ref{lem:gaussian_norm_concentration} and \(\|\boldsymbol{\mu}_{C_k}\|\le R\) imply that, with probability at least \(1-\delta/2\),
\begin{equation}
\max_{1\le k\le M}\|P\mathbf{x}_T^{(k)}\|
\le
\alpha_T R+\nu_T\Bigl(\sqrt r+\sqrt{2\log\frac{2M}{\delta}}\Bigr).
\end{equation}

On the intersection of these two events, whose probability is at least \(1-\delta\), Proposition~\ref{prop:deterministic_pairwise_alignment} gives, for every pair \((a,b)\), some orthogonal matrix \(\mathbf{H}_{a,b}\in O(d)\) such that
\begin{equation}
\begin{aligned}
\max_{0\le n\le N}
\|\mathbf{D}_n^{(a)}&-\mathbf{H}_{a,b}\mathbf{D}_n^{(b)}\| \\
&\le
A_T\cdot 2\nu_T\sqrt{2\log\frac{4M}{\delta}} \\
&+
A_T\cdot 2\Bigl(\alpha_T R+\nu_T\Bigl[\sqrt r+\sqrt{2\log\frac{2M}{\delta}}\Bigr]\Bigr)
+
2B_T R.
\end{aligned}
\end{equation}
Taking the maximum over \((a,b)\) and then the infimum over \(\mathbf{H}\in O(d)\) yields the stated inequality.
\end{proof}

\begin{proof}[Proof of Corollary~\ref{cor:main_consistency}]
Divide the bound in Theorem~\ref{thm:explicit_consistency} by \(\nu_T\sqrt d\). This gives
\begin{equation}
\begin{aligned}
\frac{1}{\nu_T\sqrt d}
\max_{1\le a,b\le n}\inf_{\mathbf{H}\in O(d)}&\max_{0\le \ell\le L}
\|\mathbf{D}_\ell^{(a)}-\mathbf{H} \mathbf{D}_\ell^{(b)}\| \\
&\le
2\sqrt{2}\,A_T\sqrt{\frac{\log(4M/\delta)}{d}}
+
2A_T\sqrt{\frac{r}{d}} \\
&+
2A_T\sqrt{\frac{2\log(2M/\delta)}{d}}
+
2A_T\frac{\alpha_T R}{\nu_T\sqrt d}
+
2B_T\frac{R}{\nu_T\sqrt d}.
\end{aligned}
\end{equation}
Since \(r+\log n=o(d)\) and \(R=o(\nu_T\sqrt d)\), the right-hand side converges to \(0\) for every fixed \(\delta\in(0,1)\). This proves the claim.
\end{proof}

\subsection{Proofs of the low-dimensionality Corollary}
\label{app:low_dimensionality_proofs}

We now show that under mild assumptions, each PF-ODE trajectory is confined to a low-dimensional subspace. Moreover, a simple alignment brings all trajectories into a single shared subspace.

\begin{corollary}[Explicit version of Corollary~\ref{cor:trajectory_low_dim}]
\label{cor:appendix_single_traj_low_dim}
Suppose \(\mathbf{x}_{t\in[0,T]}^{(k)}\) is an exact PF-ODE trajectory. Then
\begin{equation}
\mathbf{x}_t^{(k)}\in \mathcal S_k:=U\oplus \operatorname{span}\{Q\mathbf{x}_T^{(k)}\},
\qquad
\forall t\in[0,T].
\end{equation}
In particular,
\begin{equation}
\dim(\mathcal S_k)\le r+1.
\end{equation}
Hence, on any discrete reverse-time grid \(T=t_0>t_1>\cdots>t_L=0\), the trajectory matrix
\begin{equation}
\mathbf{X}^{(k)}\in \mathbb R^{(L+1)\times d},
\qquad
\text{row}_\ell(\mathbf{X}^{(k)})=(\mathbf{x}_\ell^{(k)})^\top,
\end{equation}
satisfies
\begin{equation}
\operatorname{rank}(\mathbf{X}^{(k)})\le r+1.
\end{equation}
Equivalently, the empirical trajectory covariance
\begin{equation}
\mathbf{\Sigma}_{\mathrm{traj}}^{(k)}
:=
\frac{1}{M+1}\sum_{m=0}^M
\bigl(\mathbf{x}_\ell^{(k)}-\bar{\mathbf{x}}^{(k)}\bigr)\bigl(\mathbf{x}_\ell^{(k)}-\bar{\mathbf{x}}^{(k)}\bigr)^\top
\end{equation}
has rank at most \(r+1\), and therefore
\begin{equation}
\lambda_j\!\left(\mathbf{\Sigma}_{\mathrm{traj}}^{(k)}\right)=0,
\qquad
j\ge r+2.
\end{equation}
\end{corollary}

\begin{proof}
By Lemma~\ref{lem:exact_decomposition},
\begin{equation}
Q\mathbf{x}_t^{(k)}=\rho_t Q\mathbf{x}_T^{(k)},
\end{equation}
which lies in the one-dimensional subspace \(\operatorname{span}\{Q\mathbf{x}_T^{(k)}\}\). On the other hand,
\begin{equation}
P\mathbf{x}_t^{(k)}\in U
\end{equation}
for all \(t\in[0,T]\). Therefore,
\begin{equation}
\mathbf{x}_t^{(k)}=P\mathbf{x}_t^{(k)}+Q\mathbf{x}_t^{(k)}\in U\oplus \operatorname{span}\{Q\mathbf{x}_T^{(k)}\}.
\end{equation}
This proves that the whole trajectory lies in the linear subspace \(\mathcal S_k\) with \(\dim(\mathcal S_k)\le r+1\). Since every row of \(\mathbf{X}^{(k)}\) belongs to \(\mathcal S_k\), we have
\begin{equation}
\operatorname{rank}(\mathbf{X}^{(k)})\le \dim(\mathcal S_k)\le r+1.
\end{equation}
Since $\mathcal S_k$ is a linear subspace, it contains the centered trajectory vectors. Thus, the image of the covariance matrix $\mathbf{\Sigma}_{\mathrm{traj}}^{(k)}$ lies within $\mathcal S_k$, implying
\begin{equation}
\operatorname{rank}\!\left(\mathbf{\Sigma}_{\mathrm{traj}}^{(k)}\right)\le r+1.
\end{equation}
This directly leads to the conclusion about the eigenvalues.
\end{proof}

\begin{corollary}[Explicit version of Corollary~\ref{cor:low_dimensional_trajectory_alignment}]
\label{cor:appendix_common_low_dim}
Fix any unit vector \(\mathbf{e}\in U^\perp\). For each trajectory \(k\), choose an orthogonal matrix \(\mathbf{H}_k\in O(d)\) such that
\begin{equation}
\mathbf{H}_k|_U = \mathbf{I},
\qquad
\mathbf{H}_k Q\mathbf{x}_T^{(k)}=\|Q\mathbf{x}_T^{(k)}\|\,\mathbf{e}.
\end{equation}
Define the aligned trajectory
\begin{equation}
\tilde{\mathbf{x}}_t^{(k)}:=\mathbf{H}_k\mathbf{x}_t^{(k)}.
\end{equation}
Then
\begin{equation}
\tilde{\mathbf{x}}_t^{(k)}\in \widetilde{\mathcal S}:=U\oplus \operatorname{span}\{\mathbf{e}\},
\qquad
\forall k,\ \forall t\in[0,T].
\end{equation}
In particular,
\begin{equation}
\dim(\widetilde{\mathcal S})\le r+1.
\end{equation}
Hence, if \(\widetilde{\mathbf{X}}\in\mathbb R^{n(L+1)\times d}\) denotes the pooled aligned trajectory matrix obtained by stacking all aligned trajectory states, then
\begin{equation}
\operatorname{rank}(\widetilde{\mathbf{X}})\le r+1.
\end{equation}
Equivalently, the pooled covariance of the aligned trajectories has rank at most \(r+1\).
\end{corollary}

\begin{proof}
By construction, \(\mathbf{H}_k\) acts as the identity on \(U\). Therefore, since \(P\mathbf{x}_t^{(k)}\in U\),
\begin{equation}
\mathbf{H}_k P\mathbf{x}_t^{(k)}=P\mathbf{x}_t^{(k)}\in U.
\end{equation}
Moreover, using Lemma~\ref{lem:exact_decomposition},
\begin{equation}
Q\mathbf{x}_t^{(k)}=\rho_t Q\mathbf{x}_T^{(k)},
\end{equation}
and thus
\begin{equation}
\mathbf{H}_k Q\mathbf{x}_t^{(k)}=\rho_t \mathbf{H}_k Q\mathbf{x}_T^{(k)}=\rho_t \|Q\mathbf{x}_T^{(k)}\|\,\mathbf{e} \in \operatorname{span}\{\mathbf{e}\}.
\end{equation}
Combining the \(U\)-part and the \(U^\perp\)-part gives
\begin{equation}
\tilde{\mathbf{x}}_t^{(k)}=\mathbf{H}_k\mathbf{x}_t^{(k)}\in U\oplus \operatorname{span}\{\mathbf{e}\}=\widetilde{\mathcal S}.
\end{equation}
Therefore all aligned trajectories share a single subspace $\widetilde{\mathcal S}$ of dimension at most $r+1$, which immediately yields the rank results.
\end{proof}

\begin{remark}[Connection to least-squares reconstruction]
\label{rem:ls_connection}
The above corollaries provide a direct explanation for the least-squares residual observed in practice. Specifically, every state $\mathbf{x}^{(k)}_n$ on a given trajectory $k$ belongs to the same linear subspace \(\mathcal S_k\) with
\begin{equation}
\dim(\mathcal S_k)\le r+1.
\end{equation}
Therefore, if a history buffer
\begin{equation}
\{\mathbf{x}_{n-B}^{(k)},\mathbf{x}_{n-B+1}^{(k)},\dots,\mathbf{x}_{n-1}^{(k)}\}
\end{equation}
spans \(\mathcal S_k\), then the least-squares reconstruction error of \(\mathbf{x}_n^{(k)}\) from this buffer is exactly zero:
\begin{equation}
\min_{\mathbf{c}\in\mathbb R^B}
\left\|
\mathbf{x}_n^{(k)}-\sum_{j=1}^{B}c_j \mathbf{x}_{n-j}^{(k)}
\right\|=0.
\end{equation}
Thus, in the exact spherical GMM setting, a sufficiently rich short buffer already determines the future point by linear regression. In realistic diffusion models we do not expect this identity to hold exactly, but small least-squares residuals strongly suggest that the denoising dynamics remain approximately low-dimensional.
\end{remark}

\begin{table}[t]
\centering
\caption{LEAP combined with different base solvers in SD3.5-Medium. All results use the same 15-step schedule and $r=5$ in LEAP's setting.}
\label{tab:ablation_solver}
\resizebox{\columnwidth}{!}{
\begin{tabular}{l|ccc}
\toprule
Solver + Method  & HPSv2$\uparrow$ & ImageReward$\uparrow$ & CLIPScore$\uparrow$\\
\midrule
Euler   & 27.4725  &   0.8416    &    33.4644   \\
\textbf{LEAP }(Euler) & \textbf{28.1826}  &   \textbf{0.9522}     &   \textbf{33.6817}    \\
DPM++            & 27.8642   &  0.8438     &   33.4704   \\
\textbf{LEAP} (DPM++)     & \textbf{28.2152}  & \textbf{0.9806}     & \textbf{33.5617}  \\
iPNDM         & 27.9220   & 0.8770   &  33.4416  \\
\textbf{LEAP} (iPNDM)             & \textbf{28.2204}  &  \textbf{0.9642}   &  \textbf{33.6601} \\
UniPC         & 27.8680  &   0.8788     &   33.4703   \\
\textbf{LEAP} (UniPC)             & \textbf{28.2309}  & \textbf{0.9800}   & \textbf{33.6077}  \\
\bottomrule
\end{tabular}}
\end{table}

\section{Experimental Results}
\label{app:experimental_results}
\subsection{Ablation Study}
\subsubsection{Compatibility with Different Base Solvers.} To further evaluate the robustness and generalization ability of LEAP, we conduct an additional ablation experiment by integrating LEAP with various widely used ODE solvers, including Euler, DPM++, iPNDM, and UniPC. This experiment is conducted on SD3.5 with NFE=$15$ and $r=5$. As shown in Table~\ref{tab:ablation_solver}, LEAP improves human preference and text-image alignment metrics (HPSv2, ImageReward, and CLIPScore) across all tested baseline solvers.

\begin{table}[t]
  \centering
  \caption{Ablation on the prediction-to-correction (P:C) schedule on
  SD3.5-Medium.}
  \label{tab:pc_schedule}
  \resizebox{\columnwidth}{!}{
  \begin{tabular}{ccccc}
    \toprule
    NFE & P:C & HPSv2$\uparrow$ & ImageReward$\uparrow$ & CLIPScore$\uparrow$ \\
    \midrule

    \multirow{2}{*}{10}
      & $1{:}1$ & \textbf{26.8419} & \textbf{0.7692} & \textbf{33.2982} \\
      & $1{:}2$ & 26.3182 & 0.7224 & 33.1515 \\

    \midrule
    \multirow{2}{*}{15}
      & $1{:}1$ & \textbf{28.1826} & \textbf{0.9522} & 33.6817 \\
      & $1{:}2$ & 27.9594 & 0.9004 & \textbf{33.7109} \\

    \midrule
    \multirow{2}{*}{20}
      & $1{:}1$ & \textbf{28.5727} & \textbf{0.9701} & 33.7192 \\
      & $1{:}2$ & 28.3575 & 0.9499 & \textbf{33.8125} \\

    \bottomrule
  \end{tabular}}
\end{table}

\subsubsection{Predict-correct schedule.} We further test the ratio between prediction and correction steps. Consecutive linear predictions are theoretically degradate to our current implementation as the same linear subspace is used. We thus add results regarding 2-step consecutive corrections below, which are inferior to our default 1:1 setting, as shown in Table~\ref{tab:pc_schedule}.

\subsection{Robustness, Transferability, and Generalization}

\subsubsection{Prompt-distribution robustness.}
The coefficient bank is constructed using reference prompts sampled from DrawBench, while evaluation is conducted on four prompt categories from the HPSv2 benchmark. As shown in Table~\ref{tab:prompt_distribution}, LEAP improves Euler across every category for HPSv2, ImageReward, and CLIPScore. These results indicate that LEAP remains effective under a prompt distribution shift.

\begin{table}[ht]
  \centering
  \caption{Ablation study on prompt distributions (NFE$=15$). Reference
  prompts are sampled from DrawBench, while evaluation prompts are sampled
  from the HPSv2 benchmark.}
  \label{tab:prompt_distribution}
  \resizebox{\columnwidth}{!}{%
  \begin{tabular}{llcccc}
    \toprule
    Metric & Method & Anime & ConceptArt & Paintings & Photo \\
    \midrule
    HPSv2
      & Euler & 30.14 & 28.63 & 28.99 & 26.56 \\
     & LEAP (Euler) & \textbf{31.22} & \textbf{30.15} & \textbf{30.14} & \textbf{27.37} \\
    \midrule
    ImageReward
      & Euler & 1.10 & 1.02 & 1.06 & 0.80 \\
      & LEAP (Euler) & \textbf{1.18} & \textbf{1.15} & \textbf{1.16} & \textbf{0.96} \\
    \midrule
    CLIPScore
      & Euler & 35.18 & 33.61 & 33.41 & 32.28 \\
      & LEAP (Euler) & \textbf{35.19} & \textbf{33.96} & \textbf{33.69} & \textbf{32.79} \\
    \bottomrule
  \end{tabular}}
\end{table}

\subsubsection{Coefficient transferability.}
We evaluate whether learned coefficients can be reused beyond the solver, model, or sampling schedule from which it was constructed. Specifically, we study cross-model, cross-solver, and cross-step transfer. As reported in Table~\ref{tab:coefficient_transfer}, every transferred configuration improves its corresponding baseline on all three metrics.

\begin{table}[!ht]
  \centering
  \caption{Transferability of the coefficient bank. Cross-model and
  cross-solver results use 10 NFE; cross-step results use 15 NFE.}
  \label{tab:coefficient_transfer}
  \small
  \setlength{\tabcolsep}{3.3pt}
  \resizebox{\columnwidth}{!}{%
  \begin{tabular}{llccc}
    \toprule
    Transfer setting & Method & HPSv2$\uparrow$ & ImageReward$\uparrow$ & CLIPScore$\uparrow$ \\
    \midrule
    Cross-model
      & Euler & 25.89 & 0.67 & 32.79 \\
      & LEAP (FLUX coeffs.) & \textbf{26.97} & \textbf{0.75} & \textbf{33.26} \\
    \midrule
    Cross-solver
      & DPM++ & 26.46 & 0.72 & 32.90 \\
      & LEAP (Euler coeffs.) & \textbf{27.01} & \textbf{0.80} & \textbf{33.48} \\
    \midrule
    Cross-step
      & Euler & 27.47 & 0.84 & 33.46 \\
      & LEAP (10-NFE coeffs.) & \textbf{28.17} & \textbf{0.91} & \textbf{33.69} \\
    \bottomrule
  \end{tabular}}
\end{table}

\subsection{Additional Qalitative Results}
We first provide additional qualitative results for both text-image task and image-3D task under the same setting in Sec.4.1. For text-image task, the qualitative results are shown in Fig.~\ref{fig:appendix_flux}. For image-3D task, the qualitative results are shown in Fig.~\ref{fig:3D_appendix_results}.

\subsection{Experimental Details}
\subsubsection{Hyperparameters}
In this subsection, we discuss our experiment setup in greater detail. In Sec.~\ref{sec:experiments}, we evaluate the proposed LEAP on three diffusion models: SD3.5-Medium~\cite{esser2024scaling} and FLUX.1-Dev~\cite{flux1dev} for text-image task and TRELLIS.2~\cite{xiang2025native} for image-3D task.

We test LEAP against previous solvers in text-to-image tasks, including DPM++~\cite{lu2025dpm}, STORK~\cite{tan2025stork}, and LTC-Accel~\cite{zhu2025accelerating}. In our evaluation, DPM++ is configured to its second-order version, STORK utilizes $k=4$, $n=1$, and $s=9$, while LTC-Accel adopts $r=2$ and $\tau=0.1$ for both quantitative and qualitative assessments. We use a fixed window size of $r=5$ and set the ridge regularization coefficient to $\lambda=0.0001$ as the default experiments settings of LEAP.

In the context of image-to-3D tasks, we benchmark our approach against the Euler baseline. For these experiments, we maintain the exact same configuration as mentioned in the text-to-image tasks described above.

\subsubsection{Prompts}
In the text-to-image qualitative experiments, we adopt the following text prompts to generate images:
\begin{itemize}
    \item "\textit{A close-up of a futuristic barista robot pouring latte art in the shape of a perfect hexagon. The robot's metallic arm has the word 'SPEED' etched into it. High contrast, cinematic lighting, photorealistic.}"
    \item "\textit{A transparent glass teapot filled with blooming jasmine flowers. Steam is rising, and sunlight is refracting through the water and glass, casting complex caustics onto a white tablecloth.}",
    \item "\textit{A swarm of tiny, glowing robotic drones weaving an intricate, glowinglaser web between two giant, ancient redwood trees in a foggy forest. The web perfectly forms the exact shape of a rigid, symmetrical pentagram. Volumetric fog.}
    "
\end{itemize}

\begin{figure*}[h] 
  \centering
  \begin{subfigure}{0.32\linewidth}
    \centering
    \includegraphics[width=\linewidth]{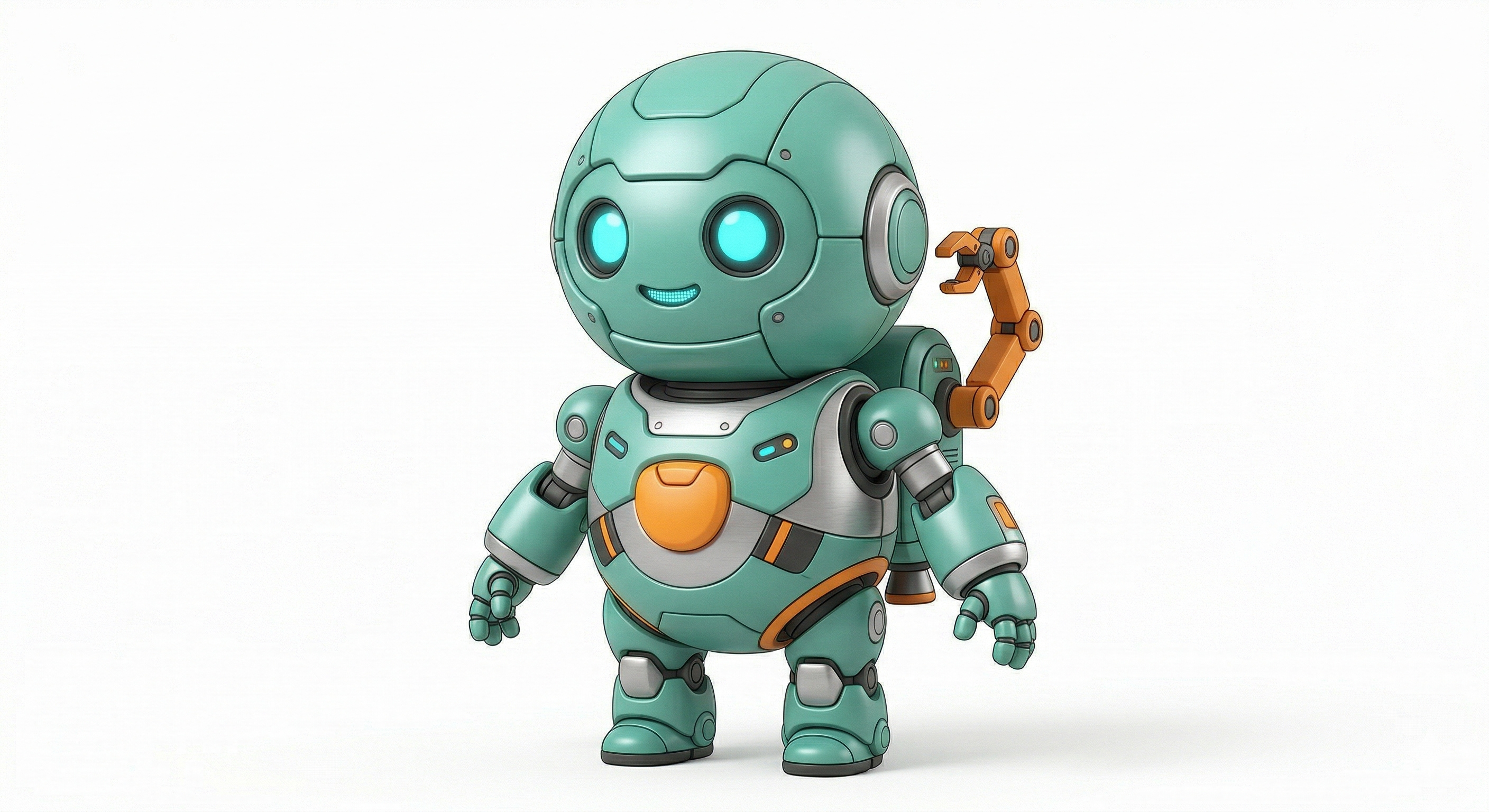} 
  \end{subfigure}\hfill
  \begin{subfigure}{0.32\linewidth}
    \centering
    \includegraphics[width=\linewidth]{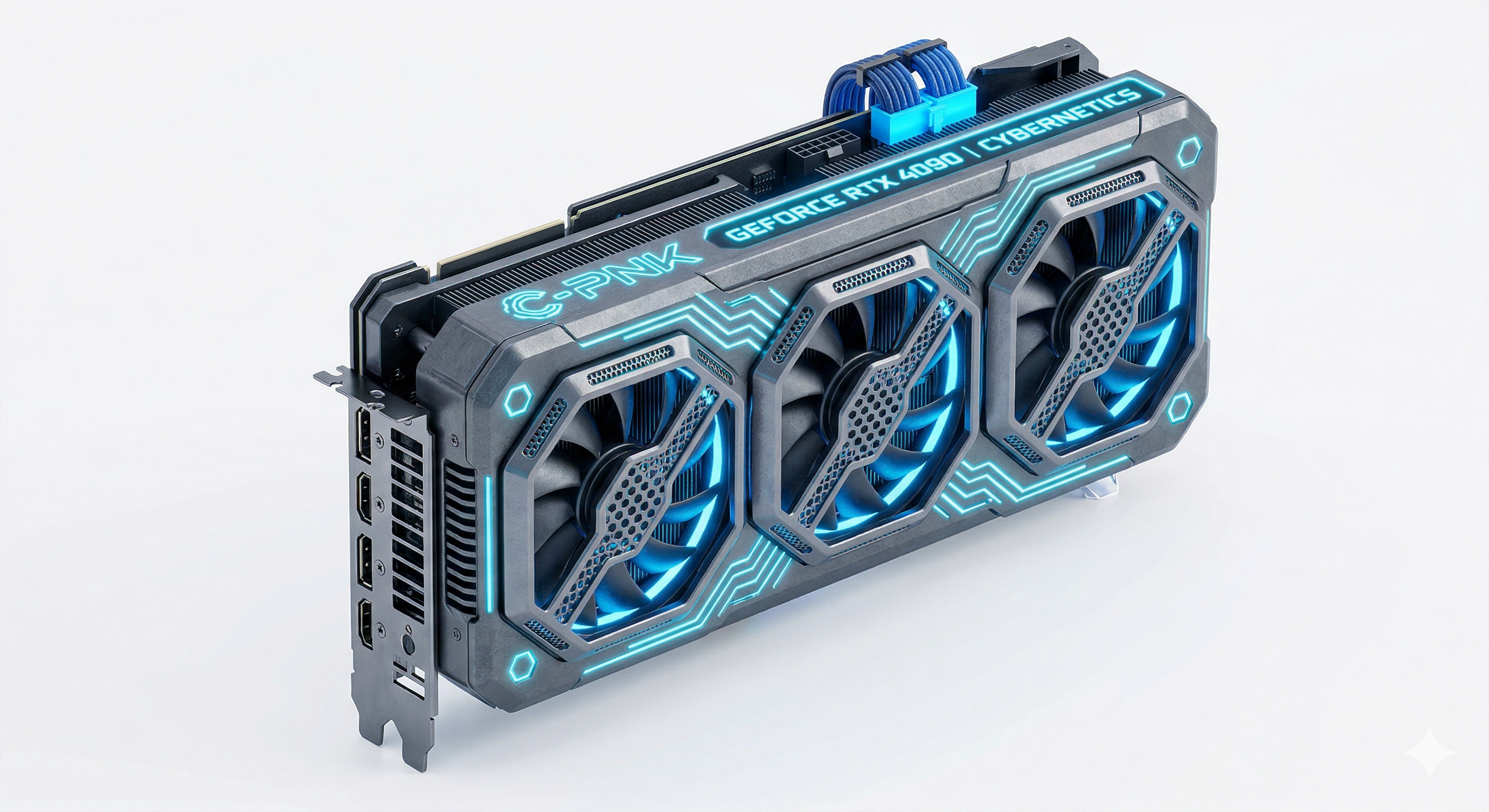}
  \end{subfigure}\hfill
  \begin{subfigure}{0.32\linewidth}
    \centering
    \includegraphics[width=\linewidth]{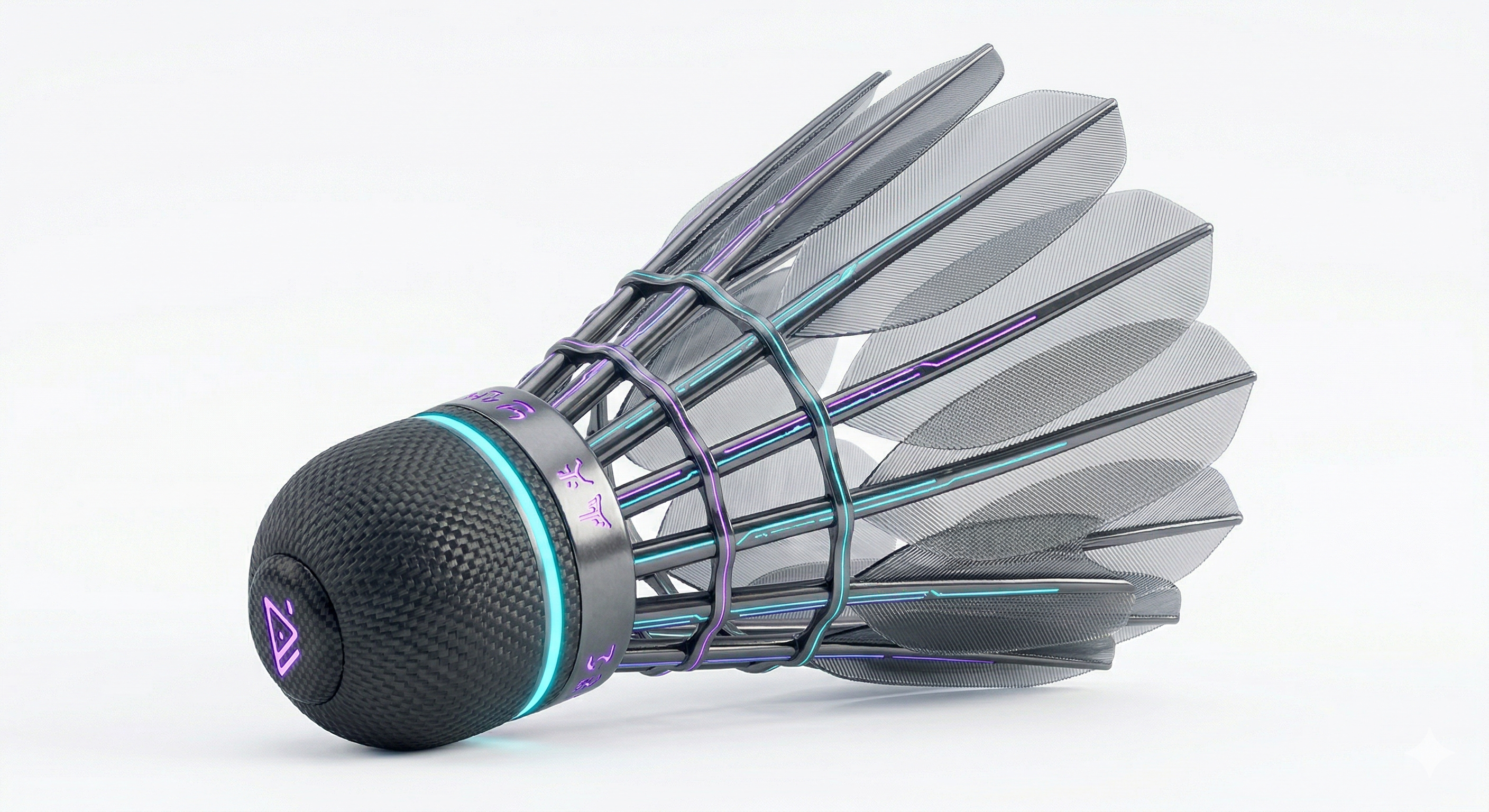}
  \end{subfigure}
  
  \caption{Reference images adopted in image-3D qualitative experiments.}
  \label{fig:reference_3d_images}
\end{figure*}

In the image-to-3D qualitative experiments, we use Nano Banana2 to generate reference images, as shown in Fig.~\ref{fig:reference_3d_images}.

\begin{figure*}[h]
    \centering
    \includegraphics{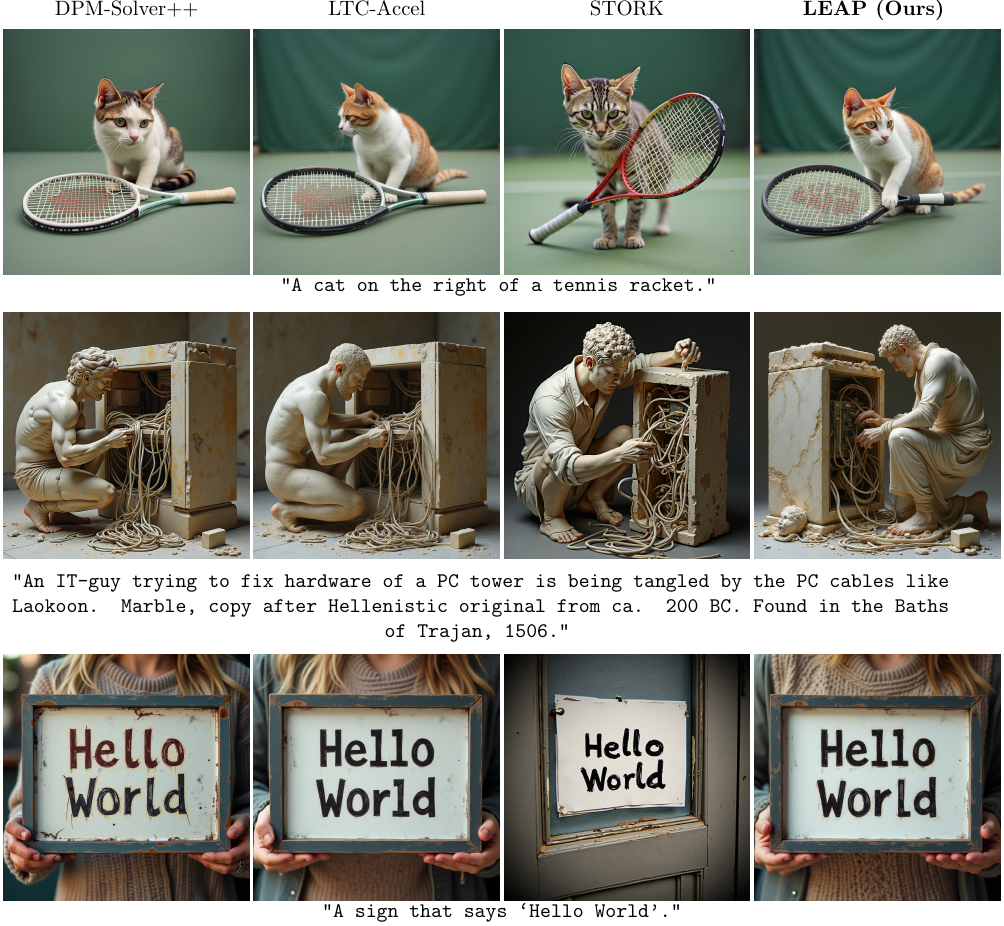}
    \caption{Additional qualitative comparison on the FLUX model (Part 1).}
    \label{fig:appendix_flux}
\end{figure*}

\begin{figure*}[htbp]
    \ContinuedFloat
    \centering
    \includegraphics{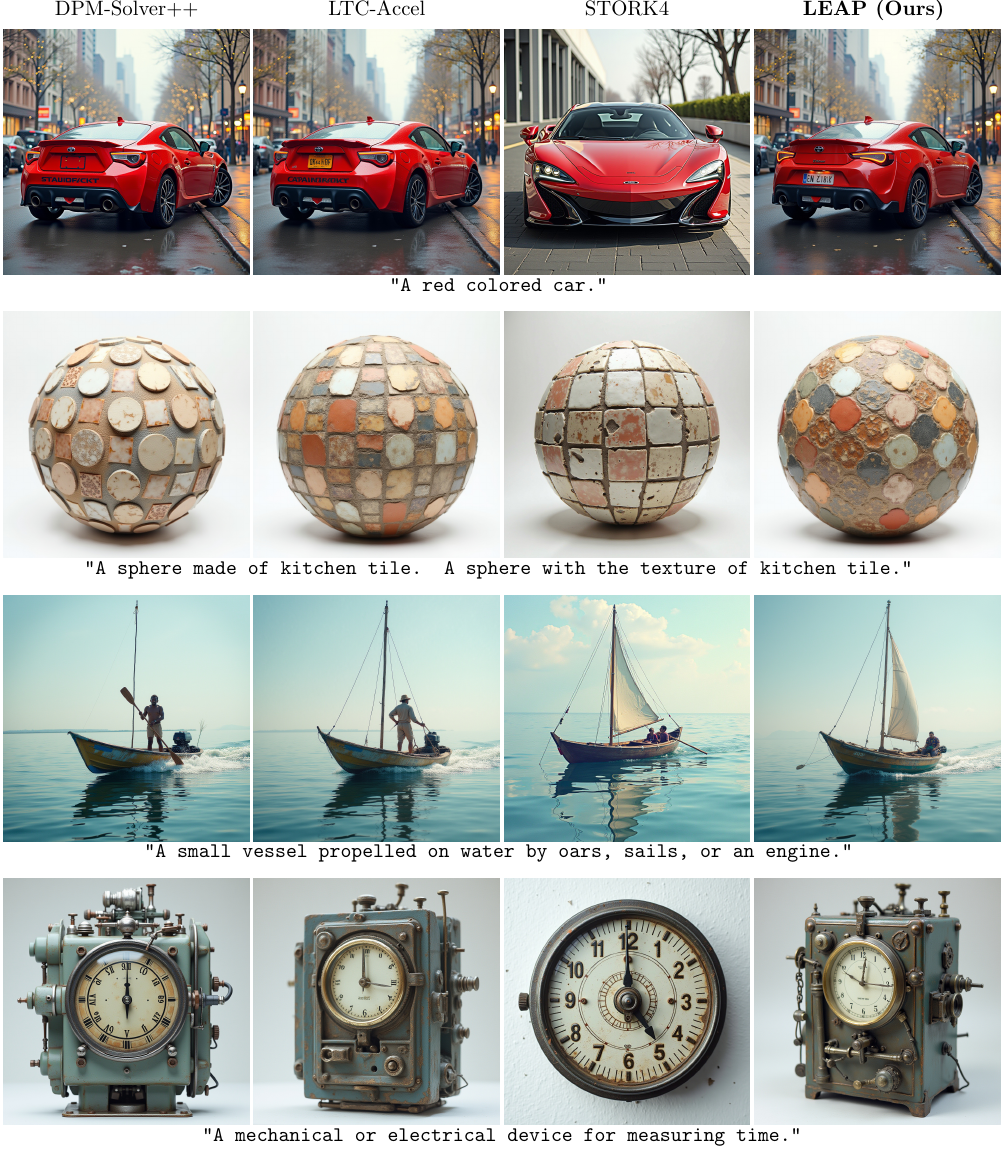}
    \caption{Additional qualitative comparison on the FLUX model (Continued).}
\end{figure*}

\begin{figure*}[htbp]
    \centering
    \includegraphics[width=0.9\textwidth]{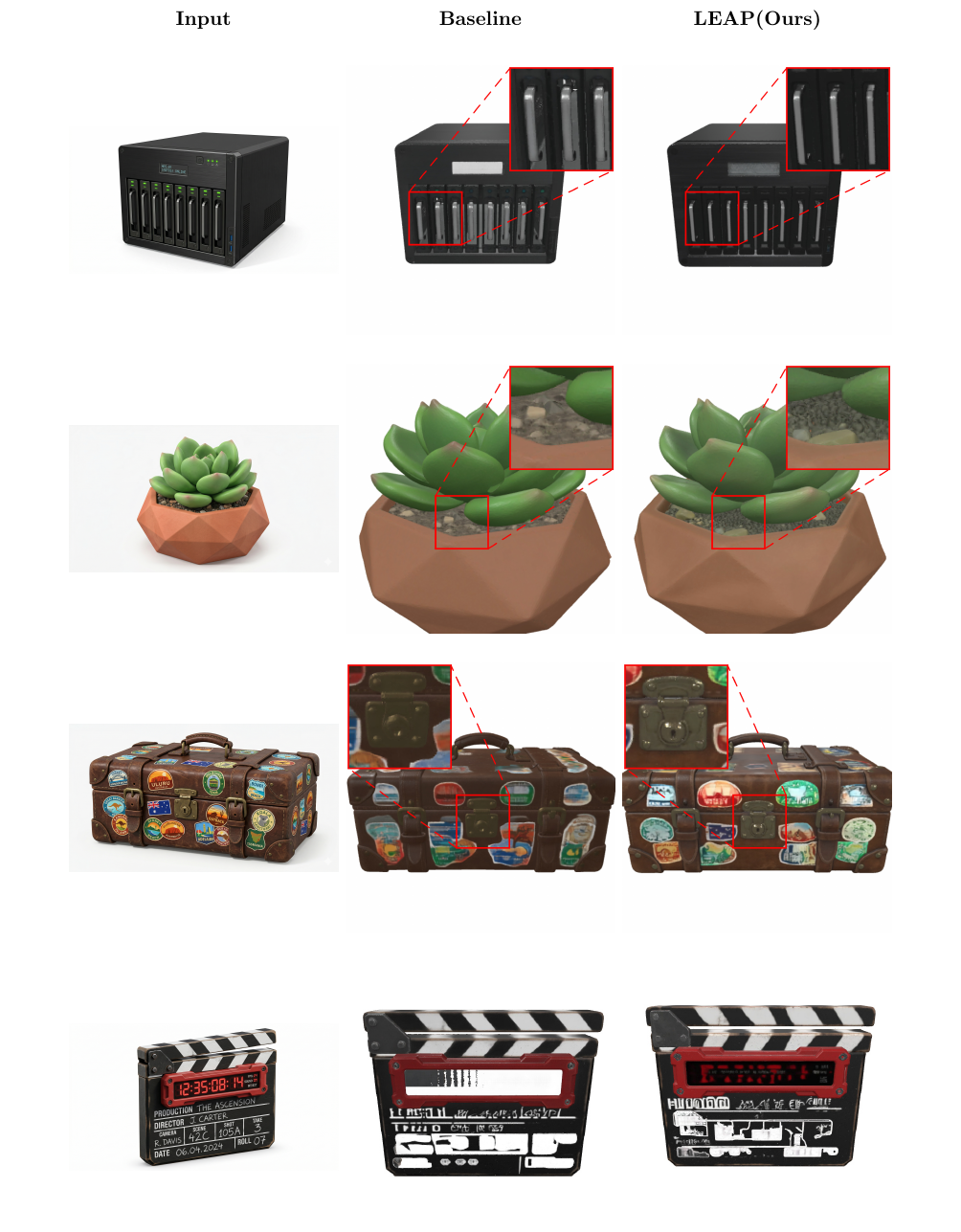}
    
    \caption{Additional image-3D qualitative results. Columns from left to right show the input conditions, generation results from the baseline solver (NFE=10), and results from our proposed LEAP method (NFE=10).}
    \label{fig:3D_appendix_results}
\end{figure*}

\end{document}